\documentclass{article} 
\usepackage{macros} 

\usepackage{pifont}

\usepackage{natbib}
\usepackage{microtype}
\usepackage{graphicx}
\usepackage{subcaption} 
\usepackage{algorithmic}
\usepackage{booktabs} % for professional tables

\title{Generalized Random Direction Newton Algorithms for Stochastic Optimization}
\author{
\;\;\; Soumen Pachal\\
{\normalsize Indian Institute of Technology Madras}\\
{\normalsize \texttt{cs22d009@smail.iitm.ac.in}}
\and
Prashanth L.A.\\
{\normalsize Indian Institute of Technology Madras}\\
{\normalsize \texttt{prashla@cse.iitm.ac.in}}
\and
Shalabh Bhatnagar\\ 
{\normalsize Indian Institute of Science, Bangalore} \\
{\normalsize \texttt{shalabh@iisc.ac.in}}
\and
Avinash Achar\\
{\normalsize TCS Research, Chennai, India} \\
{\normalsize \texttt{achar.avinash@tcs.com}}
}
\date{}
\begin{document} 
\maketitle          
\begin{abstract}                          
We present a family of generalized Hessian estimators of the objective using random direction stochastic approximation (RDSA) by utilizing only noisy function measurements.
%We incorporate continuous-valued perturbations into the Hessian estimators.
The form of each estimator and the order of the bias depend on the number of function measurements. In particular, we demonstrate that estimators with more function measurements exhibit lower-order estimation bias. We show the asymptotic unbiasedness of the estimators. We also perform asymptotic and non-asymptotic convergence analyses for stochastic Newton methods that incorporate our generalized Hessian estimators. Finally, we perform numerical experiments to validate our theoretical findings.     
\end{abstract}
% \begin{keywords}  
\textit{Keywords:} Stochastic optimization, stochastic approximation, random direction stochastic approximation (RDSA), Newton-based algorithms, gradient and Hessian estimation.                        % chosen from the IFAC 
% \end{keywords}
%\end{frontmatter}

\section{Introduction}
Optimization under uncertainty has been comprehensively studied in many engineering applications, including neural networks, reinforcement learning, finance, operations research, energy systems, and signal processing, to name a few. Uncertainty may arise due to randomness in the data or the environment. Stochastic optimization is a mathematical tool for solving this optimization problem under uncertainty. 

In this paper, we are interested in solving the optimization under uncertainty problem as a stochastic optimization problem. The primary goal is to find the minima of an objective, where the analytical form of the underlying objective is not known, instead, noisy function measurements or samples are available. The noise-corrupted measurements are accessible either through an oracle or real data. Such a setting is usually referred to as zeroth-order stochastic optimization. In this setting, the goal is to find the minima of a real-valued function $F:\mathbb{R}^d \rightarrow \mathbb{R}$. More formally, our aim reduces to finding a parameter $\theta^*$  using only noisy measurements of $F$, i.e.,
\begin{equation}
    \theta^* \in \argmin_{\theta \in \mathbb{R}^d} F(\theta). 
\end{equation}
A general stochastic approximation (SA) algorithm was first proposed by Robbins and Monro (RM) \cite{robbins1951stochastic} for finding the zeros of a function using noisy measurements. Under some conditions, the RM algorithm converges to the root of the function in an almost sure sense. This method can also be applied for gradient search, where noisy measurements of the gradient are only available. Infinitesimal perturbation analysis (IPA) \cite{fu2015handbook,ho2012perturbation} and likelihood ratio (LR) \cite{rubinstein1993discrete,l1994stochastic} are such gradient search methods that require only one function measurement. The SA algorithm finds several applications, including finding the minima of an objective function \cite{s2013stochastic, prashanth2025gradient}, finding the fixed point of a function \cite{prashanth2025gradient}, reinforcement learning \cite{sutton1998reinforcement}, neural networks \cite{gurney2018introduction}, fine-tuning Large Language models \cite{malladi2023fine}, and many more. 

The update rule of a stochastic gradient search algorithm is the following: 
%\begin{small}
\begin{equation}
    \theta(n+1) = \theta(n) - a(n) \widehat{\nabla} F(\theta(n)),
\end{equation}
%\end{small}
where $a(n),n\geq 0$, are the step sizes that satisfy stochastic approximation conditions, and $\widehat{\nabla} F(\theta(n))$ is an estimate of the true gradient $\nabla F(\theta(n))$ of the objective function $F$. The two-sided (one-sided) finite difference stochastic approximation (FDSA) or Kiefer-Wolfowitz (K-W) \cite{kiefer1952stochastic} is a gradient estimation method that requires $2d \; (d+1)$ function measurements per iteration. For large $d$, this method is computationally expensive as it perturbs along each coordinate direction one at a time. To mitigate this dimension dependency, random direction stochastic approximation (RDSA) was proposed in \cite{koronacki1975random},  \cite{kushner1978stochastic} in which perturbation variables that are uniformly distributed over the surface of a unit sphere in $\mathbb{R}^d$ are employed. This method requires only two function measurements per iteration. Many random perturbation-based methods have been studied in the literature, including smoothed functional (SF) \cite{katkovnik1972convergence}, \cite{bhatnagar2007adaptive}, RDSA with uniform and asymmetric Bernoulli perturbations \cite{prashanth2016adaptive}, simultaneous perturbation stochastic approximation (SPSA) \cite{spall1992multivariate}, and many more. All these methods deal with continuous variables. Discrete simultaneous perturbation stochastic approximation (DSPSA) is proposed in \cite{wang2011discrete}, which deals with optimization on discrete variables. Recently, a general SPSA-based method, MSPSA \cite{wang2025simultaneous}, has also been proposed for mixed variables that are both discrete and continuous. 
The key idea behind these approaches is to reduce the function measurements needed to obtain zeroth-order estimates of the gradient and higher-order derivatives, thereby achieving a desired level of accuracy.   
%\todoi{Add \cite{wang2025simultaneous} and \cite{wang2011discrete} for discrete setup. }
Amongst the most popular methods of the class of random perturbation approaches are SPSA and SF, as they are easy to implement, exhibit promising performance, and normally require just two function measurements per iteration. In order to generate the two perturbed parameters at each iterate for which the simulations are run, SPSA methods randomly perturb all the parameter components simultaneously using symmetric $ \pm 1 $-valued Bernoulli random variables. A family of generalized simultaneous perturbation stochastic approximation (GSPSA)  gradient estimators has recently been proposed in \cite{pachalyl2025generalized} that helps to significantly reduce the bias at the cost of a few more function measurements. Here, given a desired level of accuracy, one can construct a zeroth-order gradient estimator with the aforementioned few extra function measurements that achieves that level of accuracy in the gradient estimator. 

%\todoi{\st{Motivate for Newton based algorithm}}
Newton-based simultaneous perturbation algorithms have also been proposed in the literature, see \cite{spall2000adaptive}, \cite{bhatnagar2005adaptive}, \cite{bhatnagar2007adaptive}, \cite{spall2009feedback},
\cite{prashanth2016adaptive}, and 
\cite{prashanth2019random}, respectively. 
In practice, it has been observed that the stochastic gradient algorithms have many limitations, such as (i) slow convergence near optima, (ii) not being scale invariant, and (iii) being sensitive to the choice of the initial learning rate. Newton algorithms in general would typically require knowledge of the minimum eigenvalue of the Hessian for optimal convergence rate (see \cite{prashanth2016adaptive}, \cite{fabian1967stochastic}). In stochastic optimization settings, the minimum eigenvalue is typically not known. Stochastic Newton methods mitigate this eigenvalue dependence. These methods require the computation of the Hessian in addition to the gradient, and perform the following iterative update rule: 
\begin{equation}
    \theta (n + 1) = \theta (n) - a(n)\Theta \Big(\bar{\mathcal{H}}_n\Big)^{-1}\widehat{\nabla} F(\theta (n)), 
\end{equation} 
where $\widehat{\nabla} F(\theta (n))$ is an estimate of the true gradient $\nabla f(\theta)$, $a(n), n\geq 0$ are the step sizes, $\bar{\mathcal{H}}_n$ is an estimate of the true Hessian $\nabla^2 F(\theta)$, and $\Theta$ is an operator that projects  onto the set of positive definite matrices. This is needed to ensure that the algorithm proceeds along a descent direction at each step. 

We consider the RDSA-based adaptive Newton method similar to \cite{prashanth2016adaptive}. In this paper, our goal is to generalize the RDSA-based Hessian estimation method of \cite{prashanth2016adaptive} to achieve a lower-order bias but with a slight increase in the number of function measurements used in the Hessian estimator. 
 A similar generalization for gradient estimation was done recently in \cite{pachalyl2025generalized} where a family of generalized simultaneous perturbation-based gradient search estimators, including SPSA, SF, and RDSA was proposed that helps achieve a reduced bias.   
We now briefly describe our key contributions below. 
\begin{enumerate}
    \item \textbf{Generalized Hessian estimators:} We present a family of generalized RDSA-based Hessian estimators obtained from noisy function measurements. The form of each estimator differs from another in the required number of function measurements per iteration. Estimators with more function measurements result in a bias of reduced order.   
    \item \textbf{Estimation bounds:} We provide bounds on the bias and variance of our proposed Hessian estimators under standard assumptions. These bounds show that our proposed family of Hessian estimators have lower bias than 2SPSA \cite{spall2000adaptive}, 2RDSA \cite{prashanth2016adaptive} and other simultaneous perturbation-based Hessian estimators in the literature. 
    %\item \textbf{Asymptotic unbiasedness:} We prove that our proposed family of Hessian estimators is unbiased asymptotically.      
    %Unbiased estimators (asymptotic). 
    \item \textbf{Stationary convergence:} For a stochastic Newton method that incorporates our generalized Hessian estimator and the generalized gradient estimator from \cite{pachalyl2025generalized}, we establish asymptotic convergence to the set of first-order stationary points (FOSPs), i.e., points where the gradient of the objective vanishes. %We prove it for general $(2k+1)$-measurements Hessian estimators.  
    \item \textbf{Escaping saddle points:} At second-order stationary points (SOSPs), the gradient vanishes and the Hessian is positive semi-definite. Such points coincide with local minima if all saddle points are strict, cf. \cite[Chapter 7]{prashanth2025gradient}. We analyze a zeroth-order variant of the well-known cubic-regularized Newton method that incorporates our generalized Hessian estimator in conjunction with the gradient estimator from \cite{pachalyl2025generalized}. We establish a non-asymptotic bound that quantifies the convergence of the zeroth-order cubic Newton algorithm to an $\epsilon$-SOSP, which is an approximation to SOSP (see Section \ref{sec:non-asymptotic} for the details). 
    %Our non-asymptotic bound is $O(\epsilon^{-(\frac{7}{2} + \frac{1}{k})})$, when the gradient and Hessian estimators  have bias bounds of $O(\delta^k)$.
    For a given $\epsilon > 0$, there exists a $\delta > 0$ (the perturbation parameter) such that for gradient and Hessian estimators with a bias bound of $O(\delta^k)$, the non-asymptotic bound on the sample complexity is $O(\epsilon^{-(\frac{7}{2} + \frac{2}{k})})$. The latter gradient/Hessian estimation bounds are achievable with generalized estimators formed using $(2k+1)$ function measurements. To the best of our knowledge, we are the first to provide a non-asymptotic bound for a zeroth-order stochastic Newton algorithm with gradient/Hessian estimators formed using the simultaneous perturbation method.  
%    \todoi{Say something about escaping saddle points.}
    \item \textbf{Experiments:} We show the results of simulation experiments on the Rastrigin objective. Our numerical results show that the Hessian estimators with more function measurements provide better performance guarantees for a given simulation budget. 
\end{enumerate} 
%\todoi{TBD - add spall 2002, spall 2009, balanced 2016} 

\textbf{Related work}. 
%\section{Related work}
%\todoi{Add more, including Krishna Kumar and compare}   
We now compare our proposed method with closely related existing approaches for Hessian estimation. In \cite{fabian1971stochastic}, the classical FDSA-based Hessian estimation method was proposed, and this scheme requires $O(d^2)$ function measurements, making it computationally expensive in higher dimensions. To overcome this issue, a simultaneous perturbation-based Hessian estimation method, in the similar spirit of \cite{spall1992multivariate}, was proposed in \cite{spall2000adaptive}. The key advantage of this approach is that it requires only four function measurements per iteration, irrespective of $d$. A modified Hessian estimation method was proposed in \cite{zhu2002modified}, where the eigenvalues are projected to the half-plane.
%Another Hessian estimation method using feedback and weighting mechanisms was developed in \cite{spall2009feedback}.
In \cite{bhatnagar2005adaptive}, zeroth-order stochastic Newton algorithms with Hessian estimates formed using four, three, two, and one function measurements, respectively, are proposed, and in \cite{bhatnagar-prashanth}, Newton algorithms with three-simulation balanced Hessian estimators are proposed, where the Hessian inversion at each step is achieved using an iterative method. Smoothed functional-based second-order methods with one and two measurements have been explored in \cite{bhatnagar2007adaptive} with Gaussian random perturbations. The same simulations are used to estimate both the gradient and the Hessian there. More recently, an RDSA-based Hessian estimator with uniform and asymmetric Bernoulli perturbation parameters has been proposed in \cite{prashanth2016adaptive}, which requires three function measurements per iteration. Further, in \cite{prashanth2019random}, Hessian estimators based on deterministic perturbation sequence construction are presented and the algorithms analyzed. All these random perturbation-based Hessian estimation methods have bias either $O(\delta)$ or $O(\delta^2)$, where $\delta$ is the perturbation constant. In contrast, our Hessian estimators are able to achieve lower bias $O(\delta^k)$ at an additional cost of $(2k+1)$-function measurements. In \cite{balasubramanian2022zeroth}, a zeroth-order variant of the well-known cubic-regularized Newton method with the SF-based gradient and Hessian estimators has been studied under an additional smoothness assumption on the sample performance. In comparison, we do not make this assumption, and more importantly, we analyze a cubic-regularized Newton method with a generalized Hessian estimator that has a tunable bias. While our non-asymptotic bound is weaker compared to \cite{balasubramanian2022zeroth} owing to the more general setting that we consider, for large values of $k$, our bound is very close to the one in the aforementioned reference.  
%\todoi{Add}
%ADD \cite{spall2009feedback}, \cite{zhu2002modified} 

\textbf{Organization}. The rest of this article is organized as follows. In Section \ref{sec:GRDSA_estimators}, we first formulate the Hessian operator and then present Hessian estimators in multiple scenarios. % and later discuss the general $(2k+1)$-measurement Hessian estimation scheme. 
%In Section \ref{sec:unequal_truncation}, we present Hessian operators with unequal transaction order for both the special case and general cases. 
%In Section \ref{sec:Gen_Hessian_balanced}, we present the balanced version of Hessian estimators and the general form of $(2k-1)$-measurements balanced Hessian estimators.
In Section \ref{sec:asymptotic}, we analyze the asymptotic behaviour of the stochastic Newton method that incorporates our proposed Hessian estimators. In Section \ref{sec:non-asymptotic}, we provide the non-asymptotic analysis of cubic-regularized Newton method. % that incorporates our proposed Hessian estimators.
In Section \ref{sec:proofs}, we provide all the convergence proofs for all the claims in prior sections. We show the results of numerical experiments in Section \ref{sec:simul} that are seen to validate our theoretical findings. Finally, Section \ref{sec:con} presents the concluding remarks and outlines a few future research directions. 
%\todoi{Rewrite it}
\section{Generalized RDSA-based Hessian Estimation } %(Unbalanced Case) 
\label{sec:GRDSA_estimators}
%The current work is strongly motivated from \cite{pachalyl2025generalized}.
A family of generalized SPSA-based gradient estimators has been proposed in \cite{pachalyl2025generalized} that provides the flexibility of lowering estimator bias at the cost of extra function measurements. In particular, they propose two families of estimators, which generalize the one-sided SPSA estimate \cite[Section 5.3]{s2013stochastic} and the balanced SPSA estimate \cite{spall1992multivariate}, respectively. The two families are essentially based on two different infinite series expansions of the differentiation operator. Order-1 truncation of either of the infinite series expansions gives the unbalanced SPSA and balanced SPSA estimators, respectively. 

In this section, we introduce a family of novel estimators for the Hessian $\nabla^2 F$, which is based on the family of gradient estimators \cite{pachalyl2025generalized}, generalizing the unbalanced SPSA operator. The estimator bias can be made arbitrarily small (by taking into account a proportionately high order of the perturbation constant, a quantity typically less than 1) by considering a sufficiently large number of function measurements.  

Consider the multi-dimensional differentiation operator ${\mathcal D}^\beta$ %similar to \cite{pachalyl2025generalized}, 
where $\beta = [\beta_1,\beta_2 \dots \beta_d]$, $\beta_i \in \mathbb{N}_0$ (the set of non-negative integers) defined as follows:   
\[{\displaystyle {\mathcal D}^\beta F(\theta) \triangleq \frac{\partial^{|\beta|}F(\theta)}{\partial\theta_1^{\beta_1}\cdots \partial\theta_d^{\beta_d}}},\]
with $|\beta|=\beta_1+\cdots+\beta_d$.
As an example, if $\beta = [1,0,\cdots,0]$, then ${\mathcal D}^\beta F(\theta) = \frac{\partial F(\theta)}{\partial\theta_1}$. Likewise, if 
 $\beta = [2,3,0\cdots,0]$, then ${\mathcal D}^\beta F(\theta) = \frac{\partial^{5}F(\theta)}{\partial\theta_1^{2} \partial\theta_2^{3}}$.
Now let  $\Delta^\beta=\Delta_1^{\beta_1}\Delta_2^{\beta_2}\cdots\Delta_d^{\beta_d}$ and 
%$\theta^\beta=\theta_1^{\beta_1}\cdots\theta_d^{\beta_d}$,
 $\beta! =\beta_1!\beta_2!\cdots\beta_d!$. 
%for $\beta_1,\cdots,\beta_d \geq 0$. 

Using the above operator, the multivariate Taylor's expansion is as follows: 
\begin{equation}
\label{mte}
F(\theta+\delta\Delta) = \sum_{|\beta|=0}^{\infty} \frac{{\mathcal D}^\beta F(\theta)}{\beta!}(\delta\Delta)^\beta
= \hspace{-2mm} \sum_{|\beta|=0}^{\infty} \left(\frac{(\delta\Delta{\mathcal D})^\beta}{\beta!}\right)F(\theta),
\end{equation}
where $F$ is assumed to be infinitely many times continuously differentiable. 
Using the exponentiation operator and  shift operator 
$\tau_{\delta\Delta}F(\theta) \equiv F(\theta+\delta\Delta)$, we can compactly rewrite \eqref{mte} as
\[
\tau_{\delta\Delta} = \exp(\delta\Delta{\mathcal D}).
\]
Rearranging the equation above and using the log expansion (as shown in \cite{pachalyl2025generalized}), we arrive at an interesting series expansion of the derivative operator, which is given below.  
\begin{equation}
	\label{gradient_operator_unbalanced}
{\mathcal D} = \frac{1}{\delta\Delta}
\sum_{j=1}^{\infty}\frac{(\tau_{\delta\Delta} - {\mathcal I})^j}{j}(-1)^{j+1}.
\end{equation}
It is evident that truncating the above summation to just one term yields the unbalanced SPSA operator. In this sense, the above series expansion \eqref{gradient_operator_unbalanced} neatly generalizes the unbalanced SPSA operator.  
In \cite{pachalyl2025generalized}, it was shown that the gradient estimate obtained by truncating ${\mathcal D}$ to the first $k$ terms yields a progressively reduced bias of $O(\delta^{k})$, $k\geq 1$. The extra cost that one pays for this reduced bias is a progressively increasing number of function measurements, namely $k + 1$. 

In this paper, we explore such families of estimators for the Hessian.  By definition, Hessian  is obtained by applying the ${\mathcal D}$ operator \eqref{gradient_operator_unbalanced} twice. The idea now is to consider a truncated version of the ${\mathcal D}$ operator as in \cite{pachalyl2025generalized} and apply it twice on $F(\theta)$. We denote the truncated version of $\mathcal{D}$ upto first $k$ terms as $\mathcal{D}^k$, which is of the form below: 
\begin{equation}
	\label{eqn:grad_op_unbal_trun1}
{\mathcal D}^k = \frac{1}{\delta\Delta}
\sum_{j=1}^{k}\frac{(\tau_{\delta\Delta} - {\mathcal I})^j}{j}(-1)^{j+1}.
\end{equation}
On expanding the inner term, $(\tau_{\delta\Delta} - {\mathcal I})^j$, using the Binomial theorem and some rearrangement based on powers of $\tau_{\delta\Delta}$ \cite{pachalyl2025generalized}, we obtain the following form for $\mathcal{D}^k$:
\begin{equation}
\label{eqn:grad_op_unbal_trun2}
	\mathcal{D}^k = \frac{1}{\delta \Delta}\sum_{l = 0}^{k} \frac{(-1)^{1-l}c_l^k \tau_{\theta + l\delta\Delta}}{l!},
\end{equation}
where 
\begin{equation}
c_l^k = 
\begin{cases}
    \frac{1}{l} \prod \limits_{j = 0}^{l-1} (k - j) & l \geq 1, \\
	\sum_{j = 1}^k \frac{1}{j} & l = 0. \\
\end{cases}
	\label{eqn:clk}
\end{equation}
Note that since the shift operator is linear, $\mathcal{D}^k$  is also linear for any finite $k$.
%\footnote {In this paper, we mostly restrict ourselves to the case where perturbation constant and vector from both $\mathcal{D}^k$ operators are same.}
Since $\mathcal{D}^k$ needs to be applied twice, more possibilities arise here than in the gradient case. When $\mathcal{D}^k$ is applied a second time, one can use a different perturbation constant and perturbation vector (stochastically independent of the first one) as in \cite{spall2000adaptive}.  Another possibility is that 
$\mathcal{D}^k$  when applied twice, can be truncated up to a different order (say $k_1$ the first time and $k_2$ the second, for some $k_1,k_2\geq 1$. We first consider the simple case where for both truncated operators (i) the perturbation vector and constant are the same; and  (ii) both truncations are of the same order (section \ref{sec:equal_truncation}). Under these conditions, we prove an interesting bias reduction structure with increasing order of truncation.  The bias here decays polynomially with $\delta$, the perturbation constant. 
%\todoi{However, when the truncation orders are different in the Hessian estimate, the bias decay is only dependent on the lower of the two truncation orders (a strong guess based on a simpler case).   We show this rigorously in Sec.~\ref{sec:unequal_truncation}. In this section, we only analyze the truncated operator, which generalizes the unbalanced finite difference gradient estimate. }  
%In Sec. \ref{sec:Gen_Hessian_balanced} we perform a similar analysis of the Hessian estimator obtained by the truncated operator $\mathcal{D}^k$, which generalizes the balanced finite difference gradient estimate. Here also, when the order of truncation is same, the bias decay is polynomial in $\delta$ and the order of decay is roughly twice the order of truncation.  

%\todoi{In Sec. , we discuss the possibility of using $2$ different and independent perturbation vectors. The interesting bias decay reduction with order of truncation carries forward in this more general setting as well. One drawback of these estimators is that they need significantly more function measurements compared to the case of a single perturbation vector (Does it give any variance advantage?).  We note that under symmetric Bernoulli perturbations, for the unbalanced case,  the bias decay is polynomial in $\delta$ and the decay order is again the order of truncation. While in the balanced case, the decay order is twice the truncation order. }

%\todoi{Should we look for a U, V generalization of the Hessian estimate?} 
 
\subsection{Hessian estimate with equal truncation}
\label{sec:equal_truncation}

In this section, we consider a truncated Hessian operator of the form $\mathcal{H}^i = \mathcal{D}^i \circ \mathcal{D}^i$, where the operator $\mathcal{D}$ is truncated at the $i$-th position in the series \eqref{eqn:grad_op_unbal_trun1} (i.e., with $k=i$). 
We choose $\Delta = (\Delta_1, \Delta_2, \ldots, \Delta_d)^T$ to be a $d$-vector of independent standard Gaussian random variables. %Further, assume $\Delta_i$ is independent of $\Delta_j$, for $i \neq j$. 
Before tackling the general case of the Hessian estimate for an arbitrary order $i$,  we first demonstrate this calculation for smaller truncation orders for ease of illustration.
\subsubsection{Hessian estimate with $O(\delta)$ bias} 

Using the function values at $(\theta + \delta\Delta)$ and $\theta$, we form a Hessian estimate in the manner as explained below. First, let 
\begin{equation}
\label{eq:Hessian_3}
  \mathcal{D}_i^1F(\theta) =  \Delta_i \left[\frac{F(\theta + \delta\Delta) - F(\theta)}{\delta}\right]. 
\end{equation}
This estimate coincides with the two-measurement unbalanced gradient estimate proposed in \cite{s2013stochastic}. Applying the operator $\mathcal{D}_j^1$ on \eqref{eq:Hessian_3}, we obtain  
\begin{align*}
%&
    \widehat{\mathcal{H}}_{i,j}^1 F(\theta) &=\mathcal{D}_j^1\circ\mathcal{D}_i^1 F(\theta) \\
   & = \Delta_i \left[\frac{\mathcal{D}_j^1F(\theta + \delta\Delta) - \mathcal{D}_j^1F(\theta)}{\delta}\right] \\
    &= \Delta_j \Delta_i \left[\frac{F(\theta + 2 \delta\Delta) - 2 F(\theta + \delta\Delta) + F(\theta)}{\delta^2}\right]. 
\end{align*}   
% \begin{align*}
%     \mathcal{H}_{ij}^1 F(\theta) &= \Delta_i \Delta_j\\
%     &\times \left[\frac{F(\theta + 2 \delta\Delta) - 2 F(\theta + \delta\Delta) + F(\theta)}{\delta^2}\right].
% \end{align*} 
Using matrix notation, we have 
\begin{align*}
    \widehat{\mathcal{H}}^1 F(\theta) &= \Delta \Delta^T\left[\frac{F(\theta + 2 \delta\Delta) - 2 F(\theta + \delta\Delta) + F(\theta)}{\delta^2}\right].
\end{align*}
Using Taylor series expansion of $F(\theta + 2 \delta\Delta)$ and $F(\theta + \delta\Delta)$, we obtain
\begin{align}
    \widehat{\mathcal{H}}^1 F(\theta) &= \frac{\Delta \Delta^T}{\delta^2}\left[\delta^2 \Delta^T \nabla^2 F(\theta) \Delta + \mathbf{1}_{d \times d} O(\delta^3)\right]\nonumber\\
    &= (\Delta \Delta^T) [ \Delta^T\nabla^2 F(\theta)\Delta] + \mathbf{1}_{d \times d}O(\delta), \label{eq:s12}
\end{align}
where $\mathbf{1}_{d \times d}$ denotes the $d \times d$ matrix with all entries one. 
%To obtain a Hessian estimate, we apply $\mathcal{D}_i^1$ on $F(\theta + \delta \Delta)$ and $F(\theta)$, but this will not provide true Hessian.
Taking expectation over $\Delta$, the first term on the RHS above does not simplify to $\nabla^2 F(\theta)$. More precisely,
 $\E[(\Delta \Delta^T)(\Delta^T\nabla^2 F(\theta)\Delta) ] =  2\nabla^2 F(\theta) + \textrm{tr}(\nabla^2 F(\theta))\mathit{I}_{d}$, where $\textrm{tr}(A)$ denotes the trace of a matrix $A$ and $\mathit{I}_{d}$ denotes the $d \times d$ identity matrix. Note that this will not provide the true Hessian. 

To overcome the bias in the first term of \eqref{eq:s12}, we propose a slightly modified Hessian estimator $\mathcal{H}^1$ by using $(\Delta \Delta^T - \mathit{I})$ instead of $\Delta \Delta^T$ to recover the true Hessian. The modified Hessian estimator is given as follows: 
%By applying $\mathcal{D}^1$ on both $F(\theta + \delta\Delta)$ and $F(\theta)$, we obtain the following three measurements Hessian estimator:
\begin{align*}
    \mathcal{H}^1 F(\theta) &= (\Delta \Delta^T - \mathit{I})
     \left[\frac{F(\theta + 2 \delta\Delta) - 2 F(\theta + \delta\Delta) + F(\theta)}{\delta^2}\right].
\end{align*}  
%Here we employ a slightly modified Hessian $\mathcal{H}^1$ estimator with scaling matrix $(\Delta \Delta^T - \mathit{I})$ to recover the true Hessian. 
% Using Taylor series expansion of $F(\theta + 2 \delta\Delta)$ and $F(\theta + \delta\Delta)$ followed by straightforward simplifications, we obtain
% \begin{align*}
%     \mathcal{H}^1 F(\theta) &= \frac{(\Delta \Delta^T - \mathit{I})}{\delta^2}\left[\delta^2 \Delta^T \nabla^2 F(\theta) \Delta + O(\delta^3)\right]\\
%     &=  \nabla^2 F(\theta) + \mathbf{1}_{d \times d}O(\delta), 
% \end{align*}
% where $\mathbf{1}_{d \times d}$ denotes the $d \times d$ matrics with entries 1. 
\subsubsection{ Hessian estimate with $O(\delta^2)$ bias}
Three measurements GRDSA is defined as follows: %An order-$2$ truncation of the operator $\mathcal{D}$ yields   
\begin{equation} \label{eq:gspsa_3}
    \mathcal{D}_i^2 F(\theta) = \Delta_i \left[\frac{4F(\theta + \delta \Delta) - 3F(\theta) - F(\theta + 2\delta \Delta)}{2\delta}\right]. %\nonumber
\end{equation}
Applying $\mathcal{D}_j^2$ on \eqref{eq:gspsa_3}, we obtain %the five-measurement Hessian estimator 
\begin{align*}
%&
    \widehat{\mathcal{H}}_{ij}^2 F(\theta) &=\mathcal{D}_j^2 \circ \mathcal{D}_i^2 F(\theta) \\
    &= \Delta_i\left[\frac{4\mathcal{D}_j^2 F(\theta + \delta \Delta) - 3\mathcal{D}_j^2 F(\theta) - \mathcal{D}_j^2 F(\theta + 2 \delta \Delta)}{2\delta}\right]\\
    & = \frac{\Delta_j \Delta_i }{4 \delta^2 } \Big[22F(\theta + 2 \delta \Delta) - 24 F(\theta + \delta \Delta) -8F(\theta + 3\delta \Delta) + F(\theta + 4 \delta \Delta) + 9F(\theta)\Big]. 
\end{align*} 
%The above expression can be represented in combined form as follows: 
In matrix form, we have
\begin{align*}
&
    \widehat{\mathcal{H}}^2 F(\theta) = \frac{\Delta \Delta^T }{4 \delta^2 } \Big[22F(\theta + 2 \delta \Delta) - 24 F(\theta + \delta \Delta) -8F(\theta + 3\delta \Delta) + F(\theta + 4 \delta \Delta) + 9F(\theta)\Big]. 
\end{align*}
% In the similar manner mentioned above, we apply $\mathcal{D}^2$ on $F(\theta + 2\delta\Delta)$, $F(\theta + \delta\Delta)$ and $F(\theta)$ to obtain the following five measurements Hessian estimation: 
% \begin{align*}
% &
%     \mathcal{H}^2 F(\theta) = \frac{(\Delta \Delta^T - \mathit{I})}{4 \delta^2 } \Big[22F(\theta + 2 \delta \Delta) - 24 F(\theta + \delta \Delta)\\
%    & \qquad \qquad \qquad-8F(\theta + 3\delta \Delta) + F(\theta + 4 \delta \Delta) + 9F(\theta)\Big]. 
% \end{align*} 

Therefore, by using Taylor series expansion of $F(\theta + 2\delta\Delta)$, $F(\theta + \delta\Delta)$ and $F(\theta)$, we get
\begin{align*}
    \widehat{\mathcal{H}}^2 F(\theta) &= \frac{\Delta \Delta^T }{2! 4\delta^2}\Big[8\delta^2 \Delta^T \nabla^2 F(\theta) \Delta + \mathbf{1}_{d \times d} O(\delta^4)\Big]\nonumber\\
    &= \Delta \Delta^T [\Delta^T \nabla^2 F(\theta) \Delta] + \mathbf{1}_{d \times d} O(\delta^2).\nonumber
\end{align*} 
As discussed above, the first term above is not an unbiased estimate of the true Hessian. To mitigate this issue, we use the following modified Hessian estimator:
\begin{align*}
&
    \mathcal{H}^2 F(\theta) = \frac{(\Delta \Delta^T - \mathit{I})}{4 \delta^2 } \Big[22F(\theta + 2 \delta \Delta) - 24 F(\theta + \delta \Delta) -8F(\theta + 3\delta \Delta) + F(\theta + 4 \delta \Delta) + 9F(\theta)\Big]. 
\end{align*}
\subsubsection{Hessian estimate with $O(\delta^3)$ bias}
Four measurements GRDSA is defined as follows:
\begin{align}
& 
    \mathcal{D}_i^3 F(\theta) = \frac{\Delta_i}{6\delta } \Big[2F(\theta + 3\delta \Delta) - 9F(\theta + 2\delta\Delta)  + 18F(\theta + \delta\Delta)- 11F(\theta) \Big]. \nonumber
\end{align}
Applying $\mathcal{D}_j^3$ on the above, we get % and multiplying with the same scaling matrix $(\Delta \Delta^T - \mathit{I})$, we obtain % Thus, the seven measurements Hessian estimation is as follows: 
\begin{align*}
%&
    \widehat{\mathcal{H}}_{ij}^3 F(\theta) &=\mathcal{D}_j^3 \circ \mathcal{D}_i^3 F(\theta) \\
    &= \frac{\Delta_i}{6\delta }\Big[2\mathcal{D}_j^3 F(\theta + 3\delta \Delta) -9 \mathcal{D}_j^3 F(\theta + 2\delta \Delta) + 18\mathcal{D}_j^3 F(\theta + \delta \Delta) - 11\mathcal{D}_j^3 F(\theta )\Big] \\
    & = \frac{\Delta_j \Delta_i}{36\delta^2} \Big[4F(\theta + 6 \delta \Delta) -  36F(\theta + 5\delta \Delta)  +153F(\theta + 4\delta \Delta) - 368F(\theta + 3 \delta \Delta) + 121F(\theta)\\ 
    &\qquad \qquad \qquad\qquad +522F(\theta + 2 \delta\Delta) -396F(\theta + \delta\Delta)  
      \Big].  
\end{align*} 
%Using the scaling matrix $(\Delta \Delta^T - \mathit{I})$, 
%Thus, we obtain the following seven measurements Hessian estimator is given as follows: 
In matrix notation, we have
\begin{align*}
&
    \widehat{\mathcal{H}}^3 F(\theta)= \frac{\Delta \Delta^T}{36\delta^2} \Big[4F(\theta + 6 \delta \Delta) -  36F(\theta + 5\delta \Delta)  +153F(\theta + 4\delta \Delta) - 368F(\theta + 3 \delta \Delta) + 121F(\theta)\\ 
    &\qquad \qquad \qquad\qquad +522F(\theta + 2 \delta\Delta) -396F(\theta + \delta\Delta)  
      \Big].  
\end{align*} 
Using Taylor series expansion, we obtain
\begin{eqnarray}
    \mathcal{H}^3 F(\theta) &=& \frac{\Delta \Delta^T}{ 2!36\delta^2}\Big[72\delta^2 \Delta^T \nabla^2 F(\theta) \Delta + \mathbf{1}_{d \times d}  O(\delta^5)\Big] \nonumber\\
    &=&  (\Delta \Delta^T)[\Delta^T \nabla^2 F(\theta)\Delta] +\mathbf{1}_{d \times d}  O(\delta^3).\nonumber
\end{eqnarray} 
In a similar manner, we use the following a modified Hessian estimator: 
\begin{align*}
&
    \mathcal{H}^3 F(\theta) = \frac{(\Delta \Delta^T - \mathit{I})}{36\delta^2} \Big[4F(\theta + 6 \delta \Delta) -  36F(\theta + 5\delta \Delta)  +153F(\theta + 4\delta \Delta) - 368F(\theta + 3 \delta \Delta) + 121F(\theta)\\ 
    &\qquad \qquad \qquad\qquad \qquad +522F(\theta + 2 \delta\Delta) -396F(\theta + \delta\Delta)  
      \Big].  
\end{align*} 
\subsubsection{Hessian estimate with $O(\delta^k)$ bias}
Recall the GRDSA operator \eqref{eqn:grad_op_unbal_trun1} formed using $k+1$ function measurements is given by
\begin{equation} 
	\mathcal{D}_i^k F(\theta) = \frac{\Delta_i}{\delta}\sum_{l = 0}^{k} \frac{(-1)^{1-l}c_l^k F(\theta + l\delta\Delta)}{l!}, \label{eq:kth_grdsa}
\end{equation}
where $
c_l^k = 
\begin{cases}
    \frac{1}{l} \prod \limits_{j = 0}^{l-1} (k - j) & l \geq 1, \\
	\sum_{j = 1}^k \frac{1}{j}. & l \geq 0.
\end{cases}
$

Applying $\mathcal{D}_j^k$ on \eqref{eq:kth_grdsa}, we get% this twice, we obtain a $(2k +1)$-measurements based  Hessian estimator as follows:  

\begin{align*}
     \widehat{\mathcal{H}}_{ij}^{k} F(\theta)&= \mathcal{D}_j^k \circ \mathcal{D}_i^k F(\theta)\nonumber\\
    &= \;\Delta_i \left[\sum_{l = 0}^{k}\frac{(-1)^{1-l}c_l^k \mathcal{D}_j^k F(\theta + l\delta\Delta)}{\delta l!}\right] \\
    & = \frac{\Delta_j \Delta_i }{\delta^2}\sum_{l = 0}^k \sum_{m = 0}^{k}\hspace{-2mm}\frac{(-1)^{-l-m} c_l^k c_m^kF(\theta + (l+m)\delta\Delta)} {l! m!}.
\end{align*} 
Using matrix notation, we get
\begin{align}
&
    \widehat{\mathcal{H}}^{k} F(\theta) = \frac{\Delta \Delta^T}{\delta^2}\sum_{l = 0}^k \sum_{m = 0}^{k}\hspace{-2mm}\frac{(-1)^{-l-m} c_l^k c_m^kF(\theta + (l+m)\delta\Delta)} {l! m!}. \label{eq:G2RDSA_unbalanced}
\end{align}
As discussed earlier, $\widehat{\mathcal{H}}^k$ fails to provide true Hessian, we use slightly modified Hessian estimator $\mathcal{H}^{k}$ as follows:  
\begin{align}
&
    \mathcal{H}^{k} F(\theta) = \frac{(\Delta \Delta^T - \mathit{I})}{\delta^2}\sum_{l = 0}^k \sum_{m = 0}^{k}\hspace{-2mm}\frac{(-1)^{-l-m} c_l^k c_m^kF(\theta + (l+m)\delta\Delta)} {l! m!}. \label{eq:G2RDSA_unbalanced}
\end{align} 
To characterize the bias of the Hessian estimator, we make the below assumption. 
\begin{assumption}
\label{ass:cont_diff}
    The function $f$ is $(k+1)$-times continuously differentiable with $\mathbf{L}_{k+2}(\theta) \stackrel{\triangle}{=} \nabla^{k+2}_{i_1, i_2, \ldots, i_{k+2}}f(\theta) < \infty$, $i_1, i_2, \ldots, i_{k+2} \in \{1, 2, \ldots,d\}$, $k\geq 1$ and for any $\theta \in \mathbb{R}^d$. 
\end{assumption}  
%The result below characterizes the bias of the Hessian estimator defined above.
\begin{lemma}
Under assumption \ref{ass:cont_diff}, we have
    \begin{equation*}
        \mathcal{H}^k F(\theta) =(\Delta \Delta^T - \mathit{I})[\Delta^T \nabla^2 F(\theta) \Delta]  + \mathbf{1}_{d \times d} O(\delta^k),   
    \end{equation*} 
for any $k \geq 1$. 
    \label{lemma:kth_order}
\end{lemma}
\begin{proof} 
See Section \ref{proof:kth_order}. 
\end{proof}  

\subsubsection{Generalized Hessian Estimation with noisy measurements}
In the presence of noisy observations, the form of $(2k+1)$-measurements Hessian estimator is given as follows: 
\begin{align}
&
    \widetilde{\mathcal{H}}^k F(\theta,\xi)  = \frac{(\Delta \Delta^T - \mathit{I})}{\delta^2}  \sum_{l = 0}^k \sum_{m = 0}^{k}  \frac{(-1)^{-l-m} c_l^k c_m^k f(\theta + (l+m)\delta\Delta, \xi_l)} {l! m!}, 
    \label{eq:G2RDSA_unbalanced_noise}
\end{align}
where $\xi_l(n)$ is an i.i.d. set of random variables with mean zero. Here $f$ is the noise-corrupted sample of the performance objective and $\Delta$ is defined as before.  
\begin{remark}
In our proposed Hessian estimators, $(\Delta \Delta^T - \mathit{I})$ is used as a scaling matrix to obtain an unbiased Hessian estimator under expectation. This scaling matrix is particularly used when the perturbation parameter components are distributed according to a Gaussian distribution with zero mean. One can also consider different perturbation parameters to obtain different estimators. We now provide a few such examples below. 

\textit{RDSA-unif:} Our proposed method can be easily applied when each of the perturbation parameters is chosen from a uniform distribution over $[-\eta, \eta]$. In this case, the scaling matrix is the same as matrix $M$, defined in \cite{prashanth2016adaptive}.  

\textit{RDSA-asymp:} Our method can be used when the perturbation parameters are taken from an asymmetric Bernoulli distribution. In this case, the scaling matrix is the same as the matrix $M$ defined in \cite{prashanth2016adaptive}. 
%    For the bias lemma, we need to slightly modify the $\mathcal{D}^i$ operator. 
%	Firstly, it is convenient to place the perturbation vector in the numerator instead of the denominator. 
%	Further, we need to argue that for the bias lemma to hold for the diagonal elements of the Hessian, one needs to further massage the diagonal elements of the $\Delta \Delta^T$ matrix. This modification will be a function of the distribution of the perturbation vector as shown in \cite{prashanth2016adaptive}.    We use the same modifications in this section, where we are working with a single perturbation vector.  Consider a uniform and asymmetric Bernoulli. 
\end{remark}

\subsection{Hessian Estimate with Unequal Truncation} 
\label{sec:unequal_truncation}
%\todoi{Add description for special and general cases}
In this section, we consider unequal truncations of gradient operators to form the following Hessian truncated operator: $\mathcal{H}_{ij}^{k_1,k_2 } =\mathcal{D}_j^{k_1} \circ \mathcal{D}_i^{k_2}$, for any $k_1, k_2 \geq 1$ and $k_1 \neq k_2$. First, we demonstrate one special case (e.g. $k_1 = k, k_2 =1$) for ease of illustration. Then we consider the general case for any $k_1, k_2 > 1$. 
%\todoi{Pl fill the hand written case and try the general (k,n) (different order of truncation) case. } 

\textbf{Special Case: }
$k_1 =k$, $k_2 = 1.$

Two-measurement GRDSA is obtained as follows:
\begin{equation}
    \mathcal{D}_i^1F(\theta) = \Delta_i\left[\frac{F(\theta + \delta\Delta) - F(\theta)}{\delta}\right].
\end{equation} 
\newpage 
The $(k+3)$-measurement Hessian estimator is then obtained as follows: 
\begin{align}
%&
    \widehat{\mathcal{H}}_{ij}^k F(\theta) &=\mathcal{D}_j^k \circ \mathcal{D}_i^1 F(\theta) \nonumber\\
    & =  \Delta_i \left[\frac{\mathcal{D}_j^kF(\theta + \delta\Delta) - \mathcal{D}_j^kF(\theta )}{\delta}\right]\nonumber\\
%    &= \frac{9M_n}{4\eta^4\delta^2} \left[F(\theta + 2 \delta\Delta) +  F(\theta - 2\delta\Delta) - 2F(\theta)\right]
& = \Delta_j \Delta_i \Bigg[ \sum_{l = 0}^{k}\frac{(-1)^{1-l}c_l^k F(\theta + (l+1)\delta\Delta)}{\delta^2 l!} -  \sum_{l = 0}^{k}\frac{(-1)^{1-l}c_l^k F(\theta + l\delta\Delta)}{\delta^2 l!}\label{eq:un_special}\Bigg]. 
\end{align}
In a similar way as discussed in section \ref{sec:equal_truncation}, we have 
\begin{align}
    \mathcal{H}^k F(\theta) = 
%    & =  \Delta_j \left[\frac{\mathcal{D}_j^kF(\theta + \delta\Delta) - \mathcal{D}_j^kF(\theta )}{\delta}\right]\nonumber\\
%    &= \frac{9M_n}{4\eta^4\delta^2} \left[F(\theta + 2 \delta\Delta) +  F(\theta - 2\delta\Delta) - 2F(\theta)\right]
& (\Delta \Delta^T - \mathit{I}) \Bigg[ \sum_{l = 0}^{k}\frac{(-1)^{1-l}c_l^k F(\theta + (l+1)\delta\Delta)}{\delta^2 l!} -  \sum_{l = 0}^{k}\frac{(-1)^{1-l}c_l^k F(\theta + l\delta\Delta)}{\delta^2 l!}\label{eq:un_special}\Bigg]. 
\end{align}
\begin{lemma} 
\label{lemma: k1equalk2} 
Under assumptions \ref{ass:cont_diff}, we have
    \begin{equation*}
        \mathcal{H}^{k} F(\theta) =(\Delta \Delta^T - I)\left[\Delta^T \nabla^2 F(\theta)\Delta  \right]+ O(\delta).
    \end{equation*}
\end{lemma}
\begin{proof}
    See Section \ref{proof: k1equqlk2}. 
\end{proof}
The above result shows that the bias of the Hessian estimator defined in \ref{eq:un_special} is $O(\delta)$. Now we discuss the general case below. 

\textbf{General case: }
$k_1, k_2 \geq 1$, $k_1 \neq k_2.$
%\todoi{Add general term} 
%\subsection{Hessian Estimate with Unequal Truncation}
%In this section, we consider unequal truncations of gradient operators to form the following Hessian truncated operator: $\mathcal{H}^{k_1,k_2 } =\mathcal{D}^{k_2}\mathcal{D}^{k_1}$, for any $k_1, k_2 \geq 1$, where $(k_1 \neq k_2)$. The operator can be simplified as follows. 
\begin{align}
%   & 
   \mathcal{H}_{ij}^{k_1,k_2} F(\theta) &= \mathcal{D}_j^{k_2} \circ \mathcal{D}_i^{k_1} F(\theta)  \nonumber\\
	&\nonumber\\
	&=\frac{\Delta_i}{\delta^2  }  \left[\sum_{l = 0}^{k_2}\frac{(-1)^{1-l}c_l^k \mathcal{D}_j^{k_1} F(\theta + l\delta\Delta)}{l!}\right] \nonumber\\
	&= \frac{\Delta_i \Delta_j}{\delta^2  } \left[\sum_{l = 0}^{k_2} \sum_{m = 0}^{k_1}  \frac{(-1)^{2-l-m} c_l^{k_1} c_m^{k_2}F(\theta + (l+m)\delta\Delta)} {l! m!}\right]\nonumber\\
	&= \frac{\Delta_i \Delta_j}{\delta^2  } \Bigg[c_0^{k_1} c_0^{k_2} F(\theta)   + c_0^{k_2} \sum_{m = 1}^{k_1} \frac{(-1)^{2-m} {k_1 \choose m} F(\theta + m \delta\Delta)}{m}    + c_0^{k_1} \sum_{l = 1}^{k_2} \frac{(-1)^{2-l} {k_2 \choose l} F(\theta + l \delta\Delta)}{l}
	\nonumber\\
	& \qquad \qquad \qquad + \sum_{l = 1}^{k_1} \sum_{m = 1}^{k_2} \frac{(-1)^{2-l-m}{k_2 \choose l} {k_1 \choose m} F(\theta + (l+m)\delta\Delta)}{lm}\Bigg].  \label{eq:Hessiank1k2}
	%\label{eq:G2RDSA_unbalanced}
\end{align}
As discussed in section \ref{sec:equal_truncation}, we have 
\begin{align}
   & \mathcal{H}^{k_1,k_2} F(\theta) = \frac{(\Delta \Delta - I)}{\delta^2  } \Bigg[c_0^{k_1} c_0^{k_2} F(\theta) + c_0^{k_2} \sum_{m = 1}^{k_1} \frac{(-1)^{2-m} {k_1 \choose m} F(\theta + m \delta\Delta)}{m} \nonumber  \\
	& \qquad   + c_0^{k_1} \sum_{l = 1}^{k_2} \frac{(-1)^{2-l} {k_2 \choose l} F(\theta + l \delta\Delta)}{l}
	 + \sum_{l = 1}^{k_1} \sum_{m = 1}^{k_2} \frac{(-1)^{2-l-m}{k_2 \choose l} {k_1 \choose m} F(\theta + (l+m)\delta\Delta)}{lm}\Bigg].  \label{eq:Hessiank1k2}
	%\label{eq:G2RDSA_unbalanced}
\end{align}
%%& \qquad + c_0^{k_2} \sum_{m = 1}^{k_1} \frac{(-1)^{2-m} {k_1 \choose m} F(\theta + m \delta\Delta)}{m}  \nonumber \\
%We now analyze this Hessian operator. 
\begin{lemma} 
Under assumption \ref{ass:cont_diff}, we have 
%	{\normalsize
\begin{equation*}
        \mathcal{H}^{k_1,k_2} F(\theta) =(\Delta \Delta^T - I) \left[\Delta^T \nabla^2 F(\theta)\Delta \right] + O(\delta^k).
    \end{equation*} \newpage 
    In particular,
\begin{align}
	& \mathcal{H}^{k_1,k_2} F(\theta)\nonumber\\
    &=(\Delta \Delta^T - I)  \Bigg[ \Delta^T \nabla^2 F(\theta)\Delta  + \mathcal{O}\left(\delta^{k'}\right) +    \sum_{i=k}^{k'-1}  \Big<\nabla^{i+2} F(\theta), \underbrace{\Delta\otimes\dots  \Delta\otimes\Delta}_{i+2 \, times} \Big>  \frac{\delta^i \left( \sum_{m=1}^{k} (-1)^{1-m} {k \choose m} m^i\right)}  {(i+1)!}\Bigg], \label{eq:HessianUneqTrunc}
    \end{align}
	where $k = min(k_1,k_2)$, $k'= max(k_1,k_2)$, $k_1,k_2 \geq 1$, $k_1\neq k_2$. 
    \label{lemma:k1k2}
%	}
%for any 
\end{lemma}  
 
\begin{proof}
    See Section \ref{proof:k1k2}. 
\end{proof} 

\begin{remark}
	The term $\sum_{m=1}^{k} (-1)^{1-m} {k \choose m} m^i$ is non-zero because $i\geq k$ and hence Identity $\sum_{j = 0} ^k (-1)^{k-j} {k \choose j} j^q = 0 \textrm{ for any }0 < q < k$ cannot be used.  For instance, if $k=1$, this term is $1$  $\forall i$. If $k=2$, this term can be shown to be $2 - 2^{i}$, $\forall i \geq 2$. 
\end{remark}
\begin{remark}
	The above result shows that unequal truncation $(k_1,k_2)$ does not lead to any bias reduction advantage. The order of bias is the same as that obtained by keeping the order of truncation the same at $min(k_1,k_2)$ for both the $D$ operators of the Hessian estimate. Furthermore, recall that an equal truncation-based estimator requires $2k+1$ measurements. The unequal truncation operator would need $k_1+k_2+1$ measurements. This is easy to see from the double summation term in \eqref{eq:Hessiank1k2}. Clearly $k_1+k_2+1 \geq 2*min(k_1,k_2)+1$, which is the number of measurements needed by the equal truncation scheme.  Hence, unequal truncation $(k_1,k_2)$ also leads to relatively more function measurements. 
	Overall, there seems to be neither any statistical nor computational advantages in using unequal truncation-based Hessian estimates.  
\end{remark}

\subsection{Hessian estimation bounds}
In this section, we first provide the bias and variance bounds for our proposed Hessian estimators under assumptions that are similar to those made for simultaneous perturbation-based Hessian estimation, cf. \cite{prashanth2025gradient}. Subsequently,  we provide the asymptotic convergence guarantee for a stochastic Newton method with our proposed Hessian estimators. 

For the sake of analysis, we make the following additional assumption.
%\todoi{\st{Add assumptions}}
%\begin{assumption}
%\label{ass:cont_diff}
%    The function $f$ is $(k+1)$-times continuously differentiable with $\mathbf{L}_{k+2}(\theta) \stackrel{\triangle}{=} \nabla^{k+2}_{i_1, i_2, \ldots, i_{k+2}}f(\theta) < \infty$, $i_1, i_2, \ldots, i_{k+2} \in \{1, 2, \ldots,d\}$, $k\geq 1$ and for any $\theta \in \mathbb{R}^d$. 
%\end{assumption} 
% \begin{assumption} 
% \label{ass: perturbation_parameter}
%     The perturbation parameters $\{\Delta_n^i \}$ are i.i.d. with each $\Delta^i_n$ being independent of $\mathcal{F}_n$, $i=1,\ldots,d$, $n\geq 0$.  
% \end{assumption}

\begin{assumption} 
\label{ass:noise_bound}
    For all $n\geq 1$ and $l =0, 1,\ldots, 2k+1$, there exist $\alpha_1, \alpha_2 > 0$ such that $\E|\xi_n^l|< \alpha_1$ and $\E|f(\theta(n)+l\delta(n)\Delta(n))|^2<\alpha_2$. 
\end{assumption} 

 % \begin{assumption}
 % \label{ass:smoothness} 
 %     Suppose $f$ has $L$-Lipschitz continuous gradient, i.e.,
 %     \begin{equation}
 %         \|f(\theta_1, \xi) - f(\theta_2, \xi)\| \leq L \|\theta_1 - \theta_2\|. 
 %     \end{equation}
 % \end{assumption}

The result below provides the bias and variance for our proposed Hessian estimators in the equal truncation scenario, i.e., $k_1=k_2=k$. Let $\mathcal{F}_n = \sigma(\theta(j), j\leq n, \Delta
    (j),j<n,\xi_0(n),\ldots, \xi_{k_1}(n)), n\geq 1$ be the sigma algebra. 
\begin{lemma} [Bias Lemma]
\label{lemma:bias_unbalanced}
Under assumptions \ref{ass:cont_diff} - \ref{ass:noise_bound}, for all $i,j = 1, \ldots d$, the Hessian estimate with $O(\delta^k)$ bias defined in \eqref{eq:G2RDSA_unbalanced_noise} almost surely (a.s.) satisfies
    \begin{align*}
    &
    \Big | \E\left[\widetilde{\mathcal{H}}_{ij}^kF(\theta,\xi)|\mathcal{F}_n\right] - \nabla_{ij}^2 F(\theta) \Big |  = O(\delta^k), \; \text{and}\\
    & \E\Big\|\widetilde{\mathcal{H}}^k F(\theta,\xi) - \E\left[\widetilde{\mathcal{H}}^k F(\theta, \xi)|\mathcal{F}_n\right]\Big\|^2  = O\Big(\frac{1}{\delta^4}\Big).  
    \end{align*}
%where $C_1$ and $C_2$ are some constants. 
\end{lemma} 
\begin{proof}
 See Section \ref{proof:bias_unbalanced}. 
\end{proof} 
%\textbf{Note:} 
The above result shows that the bias of our proposed Hessian estimators is $O(\delta^k)$ when $(2k+1)$ function measurements are used, for $k\geq 1$. This is clearly superior compared to 2SPSA \cite{spall2000adaptive} and 2RDSA \cite{prashanth2016adaptive}, which only provide $O(\delta^2)$ bias with four and three function measurements, respectively.  
 
% \begin{lemma} [Variance bound]
% \label{lemma:var}
%     Let the $(2k+1)$ measurements Hessian estimator be defined in \eqref{eq:G2RDSA_unbalanced}. Then 
% Under assumptions \ref{ass:cont_diff} - \ref{ass:smoothness}, we have
%     \begin{align}%Under
%     &
%         \E \left[ \|\widetilde{\mathcal{H}}^k F (\theta, \xi, \Delta)\|^4_{F}\right] \nonumber \\
%         & \leq \sqrt{2}(d+16)^4 \nonumber \\
%         & \qquad \times \left( \frac{\mathbf{L}_{k+2}^4 \delta^{4k} (d+ 8k + 16)^{4k +8}}{(k+2)!)^4} + L^4 (d+16)^4 \right). \nonumber %(d+16)^{2k} \left( \frac{L_4^4 \delta^{12} (d+16)^6}{9} + L^4 \right) \nonumber
%     \end{align}
% Moreover, we have 
%      \begin{align}
%      &
%         \E \left[ \|\widetilde{\mathcal{H}}^k F (\theta, \xi, \Delta)\|^2_{F}\right] \nonumber \\
%         & \leq 2^{\frac{1}{4}}(d+16)^2 \nonumber\\
%         & \qquad \times \left( \frac{\mathbf{L}_{k+2}^2 \delta^{2k} (d+8k+16)^{2k+4}}{(k+2))!^2} + L^2 (d+16)^2 \right). \nonumber%(d+16)^4 \left( \frac{L_4^2 \delta^{6} (d+16)^3}{3} + L^2 \right) \nonumber
%     \end{align} 
% \end{lemma}

% \begin{proof} 
% See Section \ref{proof:var}. 
% \end{proof}
%The above result provides a variance bound for our proposed Hessian estimators for the equal truncation scenario. This bound is better than 2SF, calculated in [] under the assumption \ref{ass:smoothness}. Without the smoothness assumption \ref{ass:smoothness}, we obtain $O(\frac{1}{\delta^4})$, which is same as the bound provided in 2SPSA \cite{spall2000adaptive}, 2RDSA \cite{prashanth2016adaptive}. 
\begin{remark}
    For the unequal truncation scenario, one can obtain the bias and variance bound in a similar fashion using Lemma \ref{lemma:k1k2}. In this case, one will obtain a bias corresponding to the smaller truncation parameter. 
\end{remark}

\section{Zeroth-order stochastic Newton algorithm} 
\label{sec:asymptotic}

%\todoi{Put an algorihtm block here for two timescale}
%In this section, we provide the asymptotic convergence guarantee for the second-order Newton method with our proposed Hessian estimators.  
%\label{sec:Gen_Hessian_asymptotic} 
%\todoi{Need to remove and add Nestorov }
We consider the following update rule, similar to \cite{prashanth2025gradient}, except that we use generalized Hessian estimators that we propose here along with generalized gradient estimators.
\begin{align}
&
    \theta (n + 1) = \Gamma \Bigg(\theta (n) - a(n)\Theta \Big(\bar{\mathcal{H}}_n^k\Big)^{-1}\widehat{\nabla} f(\theta (n))\Bigg), \label{eq:update_rule}\\
    & \bar{\mathcal{H}}_n^k = \bar{\mathcal{H}}_{n-1}^k + b(n) \Big(\mathcal{H}_n^k - \bar{\mathcal{H}}_{n-1}^k\Big), \nonumber
\end{align}
where $\mathcal{H}_n^k$ is a Hessian estimate with bias $O(\delta^k)$ defined in \eqref{eq:G2RDSA_unbalanced}, for any $k\geq 1$ and $\widehat{\nabla} f(\theta (n))$ corresponds to any gradient estimator, proposed in \cite{pachalyl2025generalized}. Each generalized gradient estimate requires $(k+1)$-function measurements while a generalized Hessian estimate requires $(2k+1)$-function measurements. These function measurements are independent and can be reused. This means that only we need $(2k+1)$-function measurements to form both gradient and Hessian estimates. Here $\Gamma$ is the projection operator over a convex and compact set, and $\Theta$ corresponds to a matrix operator that projects any $d\times d$ matrix to the space of positive definite and symmetric matrices. 

For asymptotic convergence, we require the following assumptions in addition to \ref{ass:cont_diff} - \ref{ass:noise_bound}. 
\begin{assumption}
\label{ass: step_size}
    The step-size and perturbation sequences $a(n), b(n)$ and $\delta(n)$ satisfy the following conditions: 
    \begin{align}
    &
        \sum_{n = 1}^{\infty} a(n) = \sum_{n = 1}^{\infty} b(n) = \infty; \;  \lim_{n \rightarrow \infty}\Bigg(\frac{a(n)}{b(n)}\Bigg) = 0,\\
        &\sum_{n = 1}^{\infty}\Bigg(\frac{a(n)}{\delta(n)}\Bigg)^2 < \infty;\;\; \sum_{n = 1}^{\infty}\Bigg(\frac{b(n)}{\delta(n)^2}\Bigg)^2 < \infty. 
    \end{align}
\end{assumption}
\begin{assumption}
\label{ass:matrix_operator}
    Let $\{A_n\}$ and $\{B_n\}$ be any two $d \times d$ matrix sequences. If $\lim_{n\rightarrow \infty} \|A_n - B_n\| = 0$,  then $\lim_{n\rightarrow \infty} \|\Theta(A_n) - \Theta(B_n)\| = 0$. 
\end{assumption} 
\begin{assumption}
\label{ass: matrix_operator_inv}
    Let $\{C_n\}$ be a sequence of $d \times d$ matrices with $\sup_{n}\|C_n\|< \infty$. Then we have 
    \begin{align}
        \sup_{n}\|\Theta(C_n)\|< \infty; \;\;\; \sup_{n}\|(\Theta(C_n))^{-1}\|< \infty. 
    \end{align}
\end{assumption} 
\begin{assumption} 
\label{ass:stability}
    $\sup_n \|\theta(n)\|<\infty$ w.p. $1$.  
\end{assumption} 
\begin{theorem}[Asymptotic convergence] 
\label{theorem:strong_conv} 
    Under the assumptions \ref{ass:cont_diff} - \ref{ass:stability}, The iterates $\theta (n)$ according to \eqref{eq:update_rule} converge to the set $\{\theta^{'}: \nabla f(\theta^{'}) = 0\}$, as $n\rightarrow \infty$. 
    \end{theorem} 

\begin{proof} 
See Section \ref{proof:strong_conv}.   
\end{proof} 

%\section{Asymptotic Convergence} 
\section{Zeroth-order cubic-regularized Newton algorithm} 
\label{sec:non-asymptotic}
In general, the standard Newton method does not escape saddle points for non-convex objectives \cite{anandkumar2016efficient}. On the other hand, a cubic regularized Newton method \cite{nesterov2006cubic} manages to escape saddle points by incorporating a cubic term in the optimization process. We adopt this technique in a zeroth-order optimization setting, and derive a non-asymptotic bound that quantifies convergence rate to a local minimum. We present the pseudocode in Algorithm \ref{alg:CR_ZON}. 

%We begin by defining first and second-order stationary points in both deterministic and stochastic settings. Then we analyze the convergence of Algorithm \ref{alg:CR_ZON} to an $\epsilon$-second-order stationary point (SOSP). 

\begin{algorithm}
	\LinesNotNumbered
	\SetKwInOut{Input}{Input} \SetKwInOut{Output}{Output}
	\Input{initial point $\theta_0 \in \mathbb{R}^d$, 
		a non-negative sequence $\{\alpha_n\}$, and iterations $N\geq 1$;%, Probability distribution $P_R(,)$; 
	}
	\For{$n\leftarrow 1$ \KwTo $N$}{	
		Generate $m_n$ and $b_n$ measurements for gradient and Hessian estimation respectively;
        
        \tcc{Gradient estimation}
		\centerline {$ \bar{G}^k = \frac{1}{m_n}\sum_{i=1}^{m_n} \widehat{\mathcal{D}}^k F(\theta_{n-1},\xi_{i,n}) %\mathcal{D}^k (\theta_k) = A_1 %- \int_{-M_r}^{M_r} h'(1 - G^{m_k}_{\theta_{k-1}}) \widehat{\nabla} G^{m_k}_{R^{\theta_{k-1}}} dx 
			$}%\tcp*{DRM gradient estimation}
		\tcc{Hessian estimation}
		\vspace{1ex}
		%\begin{small}
			\centerline{$
				 \bar{\mathcal{H}^k}  = \frac{1}{b_n}\sum_{j=1}^{b_n} \Tilde{\mathcal{H}}^{k} F (\theta_{n-1},\xi_{j,n}), %- \int_{-M_r}^{M_r} h'(1 - G^{m_k}_{R^{\theta_{k-1}}}) \widehat{\nabla}^2 G^b_{R^{\theta_{k-1}}} dx  + \int_{-M_r}^{M_r} h{''}(1 - G^{m_k}_{R^{\theta_{k-1}}}) \left(\widehat{\nabla} G^{m_k}_{R^{\theta_{k-1}}}\right) \left(\widehat{\nabla} G^{m_k}_{R^{\theta_{k-1}}}\right)\tr  dx, 
				$}
			\vspace{1ex}
	%	\end{small} 
    where $\widehat{\cal D}^{k}_iF(\theta,\xi)=  \frac{1}{\delta\Delta_i} \sum_{l=0}^{k} \frac{(-1)^{1-l} C^{k_1}_l \{ f(\theta+l\delta\Delta, \xi_l )\}}{l!}$ \textrm{and}
    $\widetilde{\mathcal{H}}^k F(\theta,\xi)  = \frac{(\Delta \Delta^T - \mathit{I})}{\delta^2} 
     \times \sum_{l = 0}^k \sum_{m = 0}^{k}  \frac{(-1)^{-l-m} c_l^k c_m^k f(\theta + (l+m)\delta\Delta, \xi_l)} {l! m!}.$
%		where $\widehat{\mathcal{D}}^k F(.,.)$ is defined as equation (11) in \cite{pachalyl2025generalized} and $ \Tilde{\mathcal{H}}^{k}F(.,.)$ is defined in Section \ref{sec:GRDSA_estimators}\;

        \tcc{Cubic-regularized Newton step}
        \vspace{-3ex}        
		\begin{align*}
\theta_n = \argmin_{\theta \in \mathbb{R}^d} \left\{ \tilde{\rho}_h(\theta, \theta_{n-1},\bar{G}^k , \bar{\mathcal{H}^k},\alpha_n)\right\},
%    \numberthis\label{eq:subproblem}
		\end{align*}
		where
		$\tilde{\rho}_h(x,y,g,\mathcal{H},\alpha) = \left <g, x-y \right >+\frac{1}{2}\left <\mathcal{H}(x-y),x-y\right > +\frac{\alpha}{6}\|x-y\|^3$. 
	} 
    \Output{ Choose $R$ uniformly at random and return  $\theta_R $.} 
	\caption{Cubic-regularized zeroth-order stochastic Newton algorithm (CRZON)} 
	\label{alg:CR_ZON} 
\end{algorithm} 

We begin by discussing first and second-order stationary points.
In a deterministic optimization setting, a point $\bar{\theta}$ is said to be a FOSP if $\nabla F(\bar{\theta}) = 0$. However, convergence to FOSPs does not guarantee convergence to local minima as these could be saddle points. At an SOSP the gradient vanishes, and the Hessian is positive semi-definite. 
If every saddle point, say $\theta$, is strict, i.e.,  
$\nabla F(\theta)=0\textrm{ and }\lambda_{\min}(\nabla^2 F(\theta)) < 0$, then an SOSP coincides with local minima.  Moreover, it is NP-hard to distinguish between strict and non-strict saddle points, cf. \cite{anandkumar2016efficient}. Thus, SOSPs have been adopted as the goal for saddle-escaping algorithms in the non-convex optimization literature. 

In the zeroth-order optimization setting that we consider, it is difficult to find an SOSP directly as the exact form of the objective is unavailable. Instead, noisy function measurements are available. Moreover, in the non-asymptotic regime, an algorithm can find an approximate FOSP/SOSP, which are standard notions in machine learning literature. We define these quantities below. 
%A point  $\bar{\theta}$ is said to be an $\epsilon$ first-order stationary point if $\|\nabla F(\bar{\theta})\| \leq \epsilon$. 

%Now we define the $\epsilon$ first-order stationary point in the stochastic setting. 
\begin{definition} [$\epsilon$-FOSP]
    Let $\epsilon > 0$. A random point $\theta_R$ is said to be an $\epsilon$-FOSP of the objective $F$ if $\E[\|\nabla F(\theta_R)\|] \leq \epsilon$.     
\end{definition}
\begin{definition} [$\epsilon$-SOSP]
    Let $\epsilon > 0$. A random point $\theta_R$ is said to be an $\epsilon$-SOSP of the objective function $F$ if 
    \[\max \{\sqrt{\E[\|\nabla F(\theta_R)\|]}, \frac{-1}{\sqrt{\xi}}\E[\lambda_{\min}(\nabla^2 F(\theta_R))]\} \leq \sqrt{\epsilon},\] for some $\xi > 0$. Here $\lambda_{\min}(A)$ denotes the minimum eigenvalue of $A$.   
\end{definition} 
 
For our non-asymptotic analysis, we require the following additional assumption: 
\begin{assumption}
 \label{ass:smoothness_Hessian} 
     The objective $F$ has a $L_{\mathcal{H}}$-Lipschitz continuous Hessian, i.e.,
     \begin{equation}
         \|\nabla^2 F(\theta_1) - \nabla^2 F(\theta_2)\| \leq L_{\mathcal{H}} \|\theta_1 - \theta_2\|. 
     \end{equation}
 \end{assumption}
%\todoi{Theorem 6 proof needs to be typeset. }
The main result that establishes the convergence of Algorithm \ref{alg:CR_ZON} to an approximate SOSP is given below. 
\begin{theorem}[Convergence to $\epsilon$-SOSP] 
\label{theorem:non_asymptotic}
Suppose assumptions \ref{ass:cont_diff} - \ref{ass:noise_bound} and \ref{ass:stability}- \ref{ass:smoothness_Hessian} hold. Let the iterates $\{\theta_k\}$ be obtained by running Algorithm \ref{alg:CR_ZON} with the following parameters: For $n=1,\ldots,N$, set \newpage 
\begin{align} \label{eq:hyper} 
&
    \alpha_n = 3L_{\mathcal{H}} ;  \delta = \frac{\epsilon^{\frac{1}{k}}}{\max\{2C_1^2,8C_3^4\} }; N = \frac{12\sqrt{L_{\mathcal{H}}}[F(\theta_0)-F^*]}{\epsilon^{\frac{3}{2}}} \nonumber\\
    & m_n = \frac{2C_2}{\epsilon^{2+\frac{2}{k}}};\;\;\; b_n = \frac{\max\{4C_4C(d), 24C_4^2\} }{\epsilon^{1+\frac{4}{k}}} ,
\end{align}
%    \[\alpha_n = 3L_{\mathcal{H}} ;  \delta = \frac{\epsilon^{\frac{1}{k}}}{\max\{2C_1^2,8C_3^4\} }; N = \frac{12\sqrt{L_{\mathcal{H}}}[F(\theta_0)-F^*]}{\epsilon^{\frac{3}{2}}}\] \[ m_n = \frac{2C_2}{\epsilon^{2+\frac{2}{k}}};\;\;\; b_n = \frac{\max\{4C_4C(d), 24C_4^2\} }{\epsilon^{1+\frac{4}{k}}} ,\] 
    where $C_1, C_2, C_3, C_4$ are constants specified in section \ref{proof:non_asymptotic}. Then,
    \begin{equation}
        35\sqrt{\epsilon} \geq \max \left\{\sqrt{\E[\|\nabla F(\theta_R)\|]}, \frac{-1}{50\sqrt{L_\mathcal{H}}}\E[\lambda_{\min}(\nabla^2 F(\theta_R))] \right\}. 
    \end{equation} 
\end{theorem} 
\begin{proof}
See Section \ref{proof:non_asymptotic}. 
\end{proof} 
%From above result, the total sample complexity of proposed CRZON algorithm is $\sum_{i=1}^{N} m_n. (k+1) = N.m_n.(k+1) = O(\epsilon^{-(\frac{7}{2}+ \frac{1}{k})}).$ 
%\tonumberthis{\label{eq:hyper}} 
% L_{\mathcal{H}}
\begin{remark}
%The sample complexity of an algorithm using $l_n$ measurements in iteration $n$ and running for a total of $N$ iterations is  $\sum_{i=1}^{N}$. 
For finding an $\epsilon$-SOSP, the number of gradient estimates in the CRZON algorithm is $\sum_{n=1}^N m_n = O(\epsilon^{-(\frac{7}{2}+ \frac{2}{k})})$, while the number of Hessian estimates is $\sum_{n=1}^N b_n = O(\epsilon^{-(\frac{5}{2}+ \frac{4}{k})})$. Each gradient estimate uses $k+1$ function measurements and Hessian estimate requires $2k+1$ measurements. These function measurements are not required to be separate, and a subset of the ones used in forming the Hessian estimate can be reused for gradient estimation.
In any case, the overall sample complexity of CRZON algorithm is $O\left(\epsilon^{-(\frac{7}{2}+ \frac{2}{k})}\right)$. With unbiased gradient and Hessian information, cubic Newton finds an $\epsilon$-SOSP with sample complexity $O\left(\epsilon^{-3.5}\right)$, cf. \cite{tripuraneni2018stochastic}. In comparison, our bound for CRZON in the zeroth-order setting is weaker. However, this gap is expected by considering the gradients and Hessian estimates are both biased. As $k$ grows large, our estimators have lower bias, and our sample complexity bound approaches $O\left(\epsilon^{-3.5}\right)$. 
\end{remark}

\begin{remark}    
In \cite{balasubramanian2022zeroth}, where the authors consider a specialized zeroth-order setting, the corresponding bound is $O(\epsilon^{-3.5})$. The specialization in their setting is that $F(\theta)=\E_\xi[f(\theta,\xi)]$ and $f$ are both $L$-smooth. We assume $F$ alone is $L$-smooth, which is more general than \cite{balasubramanian2022zeroth}, as $L$-smoothness of $f$ implies smoothness of the objective $F$ as well while the converse is not true in general. 
\end{remark} 
\begin{remark}   
The sample complexity of a stochastic gradient algorithm to find an $\epsilon$-FOSP with unbiased gradient information is $O(\epsilon^{-4})$, cf. \cite{ghadimi2013stochastic}. On the other hand, the generalized SPSA, see \cite{pachalyl2025generalized}, with a gradient estimate formed using $k+1$ function measurements,  finds an $\epsilon$-FOSP in a zeroth-order optimization setting with sample complexity $O(\epsilon^{-(4+\frac{4}{k})})$. This sample complexity approaches $O(\epsilon^{-4})$ as $k$ becomes large. 
Our results are in a similar spirit, albeit for finding SOSPs. In particular, our bound $O(\epsilon^{-(\frac{7}{2}+\frac{2}{k})})$ for CRZON algorithm approaches $O(\epsilon^{-3.5})$  as $k$ becomes large. 
\end{remark} 
\begin{remark}    
For the special case of three measurements-based Hessian estimator, i.e., $k=1$, we obtain a sample complexity bound of $O(\epsilon^{-5.5})$ for finding an $\epsilon$-SOSP.
This result is of independent interest, as such an estimator is closer to the three measurements-based balanced Hessian estimators studied in \cite{bhatnagar2015simultaneous},\cite{prashanth2016adaptive},\cite{balasubramanian2022zeroth}. To the best of our knowledge, there are no sample complexity bounds for stochastic Newton algorithms using well-known simultaneous perturbation-based Hessian estimates in the aforementioned references. Our results close this gap and we also provide generalized Hessian estimators that can approach the sample complexity of a stochastic Newton algorithm with unbiased gradient/Hessian estimates. 
\end{remark}  
%\clearpage 
\newpage 
\section{Proofs} 
\label{sec:proofs}
\subsection{Proof of Lemma \ref{lemma:kth_order}} 
\label{proof:kth_order}
\begin{proof}
We use the following combinatorial identities in the proof: 
\begin{align}
    &\frac{1}{1} {k \choose 1} - \frac{1}{2} {k \choose 2} + \cdots + (-1)^{k + 1} \frac{1}{k} {k \choose k} = \sum_{j = 1}^k \frac{1}{j},\label{eq:ident1}\\
    &{k \choose 1} - {k \choose 2} + {k \choose 3} - \ldots (-1) ^{k+1} {k \choose k} = 1,\label{eq:ident2}    \\
 &\sum_{j = 0} ^k (-1)^{k-j} {k \choose j} j^q = 0 \textrm{ for any }0 < q < k.\label{eq:ident3}
\end{align}  
Recall that 
\begin{align*}
   & \mathcal{H}^{k} F(\theta) \\
    &=\frac{(\Delta \Delta^T - \mathit{I})}{\delta^2} \hspace{-0.4mm}\Bigg[c_0^{2k}F(\theta)\hspace{-0.4mm} +\hspace{-0.4mm} 2\sum_{m = 1}^k \hspace{-0.4mm} \frac{(-1)^{-m}c_0^k {k \choose m} F(\theta + m \delta\Delta)}{m} + \sum_{l = 1}^k \sum_{m = 1}^k \frac{(-1)^{2-l-m}{k \choose l} {k \choose m} F(\theta + (l+m)\delta\Delta)}{lm}\Bigg].
\end{align*} 
Based on Taylor's expansion, we analyze the various coefficients.
The coefficient of $F(\theta)$ can be simplified as follows. 
\begin{align}
	 &(c_0^k)^2 + 2 \sum_{m=1}^k \frac{(-1)^{2-m}c_0^k {k \choose m}}{m} + \sum_{l=1}^k \sum_{m=1}^k \frac{(-1)^{2-l-m}{k \choose l} {k \choose m}}{lm} \label{eq:temp}  \\
    =& \left(\sum_{j = 1}^k \frac{1}{j}\right)^2  + 2 \left(\sum_{j = 1}^k \frac{1}{j}\right) \left( \sum_{m=1}^k \frac{(-1)^{2-m} {k \choose m}}{m}\right)
    %(\because \text{In second term, pull out} c^0_k \text{and expand it}, \text{In double summation term, pull out terms independent of m from the inner summation.}  )
	  +  \left(\sum_{l=1}^k \frac{(-1)^{1-l} {k \choose l}}{l}\right)  \left(\sum_{m=1}^k \frac{(-1)^{1-m} {k \choose m}}{m}\right). \nonumber 
\end{align} 
In the first term of \eqref{eq:temp}, we expand $c^0_k$ from \eqref{eqn:clk}. In the second term, we pull out $c^0_k$ and expand it. In the third double summation term, we first pull out terms independent of $m$ from the inner summation. But now what is left in the inner summation with index $m$ is independent of the outer index $l$. Hence, the double summation can be simplified as a product of two single summations. 
Further using Identity \eqref{eq:ident1} in the above three summations,
	we have  
\begin{align*}
    &= \left(\sum_{j = 1}^k \frac{1}{j}\right)^2 \hspace{-1mm} - 2 \left(\sum_{j = 1}^k \frac{1}{j}\right) \left(\sum_{j = 1}^k \frac{1}{j}\right) + \left(\sum_{j = 1}^k \frac{1}{j}\right) \left(\sum_{j = 1}^k \frac{1}{j}\right)\\
	%& (\text{Use Identity \ref{eq:ident1} in }\, 3\, { summations}) \\
    & = 2 \left(\sum_{j = 1}^k \frac{1}{j}\right) \left(\sum_{j = 1}^k \frac{1}{j}\right) - 2 \left(\sum_{j = 1}^k \frac{1}{j}\right) \left(\sum_{j = 1}^k \frac{1}{j}\right)\\ 
    & = 0. 
\end{align*} 
The coefficient of $\Big<\nabla F(\theta),\Delta\Big>$ in the aforementioned Taylor's expansion is given as follows: 
\begin{align*}
 & 2 \sum_{m=1}^k (-1)^{2-m}c_0^k {k \choose m}  + \sum_{l=1}^k \sum_{m=1}^k \frac{(-1)^{2-l-m}{k \choose l} {k \choose m}}{lm} (l+m)\\
 & = 2 c_0^k  \sum_{m=1}^k (-1)^{2-m}{k \choose m} + \sum_{l=1}^k \sum_{m=1}^k  \frac{(-1)^{2-l-m}{k \choose l} {k \choose m}}{m}+ \sum_{l=1}^k \sum_{m=1}^k  \frac{(-1)^{2-l-m}{k \choose l} {k \choose m}}{l} \\
& = - 2\sum_{j = 1}^k \frac{1}{j}  + \left(\sum_{l=1}^k (-1)^{1-l}{k \choose l}\right) \left(\sum_{m=1}^k \frac{(-1)^{1-m} {k \choose m}}{m}\right) +  \left(\sum_{l=1}^k \frac{(-1)^{1-l} {k \choose l}}{l}\right) \left(\sum_{m=1}^k (-1)^{1-m}{k \choose m}\right).
\end{align*}
In the previous step, the two double summations can be written as a product of single summations, as in the previous simplification. Now using Identities \eqref{eq:ident1} and \eqref{eq:ident2} on the above single summations, we obtain   
\begin{align*}
& = - 2 \left( \sum_{j = 1}^k \frac{1}{j}\right)+ \left( \sum_{j = 1}^k \frac{1}{j}\right) +\left( \sum_{j = 1}^k \frac{1}{j}\right) \\
& = 0. 
\end{align*}
Next, the coefficient of $\Big<\nabla^2 F(\theta),\Delta \otimes \Delta\Big>$ in the Taylor's expansion is given as follows: 
\begin{align}
& \frac{1}{2!}\Big[2 \sum_{m=1}^k \frac{(-1)^{2-m}c_0^k {k \choose m}}{m}m^2  + \sum_{l=1}^k \sum_{m=1}^k \frac{(-1)^{2-l-m}{k \choose l} {k \choose m}}{lm} (l+m)^2 \Big] \nonumber \\
 & =  c_0^k \sum_{m=1}^k (-1)^{2-m} {k \choose m}m  +	\frac{1}{2}\sum_{l=1}^k \sum_{m=1}^k \hspace{-1mm} \frac{(-1)^{2-l-m}{k \choose l} {k \choose m}}{lm} l^2  + \sum_{l=1}^k \sum_{m=1}^k (-1)^{2-l-m}{k \choose l} {k \choose m} \nonumber \\
 & \qquad \qquad \qquad \qquad \qquad + \frac{1}{2}\sum_{l=1}^k \sum_{m=1}^k \frac{(-1)^{2-l-m}{k \choose l} {k \choose m}}{lm}m^2 \nonumber \\
& =  c_0^k \sum_{m=1}^k (-1)^{2-m} {k \choose m}m +	\frac{1}{2}\left(\sum_{l=1}^k (-1)^{1-l} {k \choose l} l\right) \left(\sum_{m=1}^k \frac{(-1)^{1-m} {k \choose m}}{m} \right) \nonumber \\
	& \qquad + \left(\sum_{l=1}^k (-1)^{1-l} {k \choose l} \right) \left( \sum_{m=1}^k (-1)^{1-m} {k \choose m} \right) + \left(\frac{1}{2}\sum_{l=1}^k \frac{(-1)^{1-l} {k \choose l}}{l}\right) \left(\sum_{m=1}^k (-1)^{1-m} {k \choose m}m\right) \label{eq:temp2} \\
 % & = 0+0+ 1 + 0 \\
 & = 1. \nonumber 
\end{align} 
In \eqref{eq:temp2} above, the first, third, and fourth terms vanish as a consequence of the Identity \eqref{eq:ident3} with $q=1$.  The third term is $1$ from the Identity \eqref{eq:ident2}, and hence the result.

The coefficient of $\Big<\nabla^q F(\theta), \Delta \otimes\dots  \otimes\Delta\Big>$, $3\leq q \leq k$ in the Taylor's expansion is as follows:  
\begin{align*}
& \frac{1}{q!}\Bigg[2 \sum_{m=1}^k \frac{(-1)^{2-m}c_0^k {k \choose m}}{m}m^q  + \sum_{l=1}^k \sum_{m=1}^k \frac{(-1)^{2-l-m}{k \choose l} {k \choose m}}{lm} (l+m)^q \Bigg] \\
%& = 0 + 0 \\
& = 0. 
\end{align*}  
Since $q\geq 3$, the first term above vanishes as a consequence of the Identity \eqref{eq:ident1} for $q\geq 2$. In the second term, the Binomial expansion of $(l + m)^q$ will yield $(q+1)$ terms, each having the form ${q \choose i} l^i m^{q-i}$. Since $q\geq 3$, for any $i$ s.t. $0\leq i \leq q$, $\max (i,q-i)$ will be at least $2$. This means at least one of the exponents of $l$ and $m$ will be $2$ or more in each term of the Binomial expansion. Without loss of generality, lets assume $l$ is the index whose exponent is at least $2$. The same argument will hold for $m$. After cancellation with $l$ in the denominator, the $l$'s exponent will be $1$ or more. Now this double summation term can be split into a product of single summations, where the index $l$ summation term will be of the form $\left(\sum_{l=1}^k (-1)^{1-l} {k \choose l} l^{p}\right)$, $p\geq 1$. From Identity \eqref{eq:ident3}, index $\ell$ summation term will be $0$. Hence, we have shown that every double summation term from the binomial expansion is $0$.
\end{proof} 

\subsection{Proof of Lemma \ref{lemma: k1equalk2}}
\begin{proof}
\label{proof: k1equqlk2} 
Recall that
\begin{align}
&
    \mathcal{H}^k F(\theta) =\mathcal{D}^k \circ \mathcal{D}^1 F(\theta) \nonumber\\
%    & =  \Delta \left[\frac{\mathcal{D}^kF(\theta + \delta\Delta) - \mathcal{D}^kF(\theta )}{\delta}\right]\\
%    &= \frac{9M_n}{4\eta^4\delta^2} \left[F(\theta + 2 \delta\Delta) +  F(\theta - 2\delta\Delta) - 2F(\theta)\right]
& = (\Delta \Delta^T - \mathit{I})\Bigg[ \sum_{l = 0}^{k}\frac{(-1)^{1-l}c_l^k \mathcal{D}^k F(\theta + (l+1)\delta\Delta)}{\delta^2 l!}   -  \sum_{l = 0}^{k}\frac{(-1)^{1-l}c_l^k \mathcal{D}^k F(\theta + l\delta\Delta)}{\delta^2 l!} \Bigg]. \label{eq:k1unequalk2}
\end{align}
Now we apply the Taylor's expansion to each term in \eqref{eq:k1unequalk2} and derive the coefficients of different-order derivatives. 
The coefficient of $ F(\theta)$ is given as follows: 
\begin{align*}
& = \sum_{l = 0}^k (-1)^{1-l}  c_0^k - \sum_{l = 0}^k (-1)^{1-l}  c_0^k = 0.    
\end{align*}
Next, the coefficient of $ \nabla F(\theta)$ is obtained as follows: 
\begin{align*}
& = -c_0^k + \sum_{l = 1}^k (-1)^{1-l} \frac{1}{l} {k \choose l} (l+1)    -   \sum_{l = 1}^k (-1)^{1-l} \frac{1}{l} {k \choose l} l \\
& = -c_0^k +  \sum_{l = 1}^k (-1)^{1-l} \frac{1}{l} {k \choose l} \\
 & = - \sum_{j = 1}^k \frac{1}{j} +  \sum_{j = 1}^k \frac{1}{j} = 0.   
\end{align*}
The coefficient of $ \nabla^2 F(\theta)$ is given as follows: 
\begin{align*}
& \frac{1}{2!}\Big[\hspace{-0.4mm} - c_0^k+\sum_{l = 1}^k\hspace{-0.4mm} (-1)^{1-l} \frac{1}{l} {k \choose l} (l+1)^2  - \hspace{-0.4mm}\sum_{l = 1}^k \hspace{-0.4mm} (-1)^{1-l} \frac{1}{l} {k \choose l} l^2 \Big] \\
&  = \frac{1}{2!}\Big[ -c_0^k+\sum_{l = 2}^k (-1)^{1-l} \frac{1}{l} {k \choose l} (2l+1) \Big]\\
& = \sum_{l = 2}^k (-1)^{1-l} {k \choose l}  = 1. 
\end{align*}
For any $q > 2$, the coefficient of $ \nabla^q F(\theta)$ is given as follows: 
\begin{align*}
& \frac{1}{2.q!}\left[-c_0^k  + \sum_{l = 1}^k (-1)^{1-l} \frac{1}{l} {k \choose l} (l+1)^q\right] - 
\frac{1}{2.q!}\left[c_0^k (-1)^{q+1}
\sum_{l = 2}^k (-1)^{1-l} \frac{1}{l} {k \choose l} (l-1)^q\right] \\
&  =\frac{1}{(q-1)!}. % \begin{cases}
\end{align*} 
Thus for any $q > 2$, the coefficient is not equal to zero. This completes the proof. 
%\todoi{Add the last line }
\end{proof}

\subsection{Proof of Lemma \ref{lemma:k1k2}}

\begin{proof} 
%\label{proof:k1k2}
\label{proof:k1k2}
As before, we apply Taylor's expansion to each term in \eqref{eq:Hessiank1k2} and combine the coefficients of terms with similar order derivatives in this new Hessian operator. 

The coefficient of $F(\theta)\delta^{-2}$ is the following:
	\begin{align*}
		& c_0^{k_1} c_0^{k_2}  + c_0^{k_2} \sum_{m = 1}^{k_1} \frac{(-1)^{2-m} {k_1 \choose m}}{m} 
+ c_0^{k_1} \sum_{l = 1}^{k_2} \frac{(-1)^{2-m} {k_2 \choose m} }{l}  + \left(\sum_{l=1}^{k_2} \frac{(-1)^{1-l} {k_2 \choose l}}{l}\right)  \left(\sum_{m=1}^{k_1} \frac{(-1)^{1-m} {k_2 \choose m}}{m}\right) \\
		& \overset{(a)}{=} c_0^{k_1} c_0^{k_2} + c_0^{k_2}(- c_0^{k_1}) + c_0^{k_1}(- c_0^{k_2}) + c_0^{k_1} c_0^{k_2}  \\
		& = 0, 
	\end{align*} 
    where $(a)$ follows from \eqref{eq:ident1}. 

Coefficient of $\Big<\nabla F(\theta),\Delta\Big>\delta^{-1}$ is the following:
\begin{align*}
		& c_0^{k_2} \sum_{m = 1}^{k_1} (-1)^{2-m} {k_1 \choose m} +  c_0^{k_1} \sum_{l = 1}^{k_2} (-1)^{2-l} {k_2 \choose l} +  \sum_{l = 1}^{k_1} \sum_{m = 1}^{k_2} (-1)^{2-l-m}{k_2 \choose l} {k_1 \choose m}  (1/l+1/m) \\
	& = -c_0^{k_2} -c_0^{k_1} + \hspace{-1mm} \left[\sum_{l=1}^{k_2} (-1)^{1-l}{k_2 \choose l}\right] \left[\sum_{m=1}^{k_1} \frac{(-1)^{1-m} {k_1 \choose m}}{m}\right]] +  \left[\sum_{l=1}^{k_2} \frac{(-1)^{1-l} {k_2 \choose l}}{l}\right] \left[\sum_{m=1}^{k_1} (-1)^{1-m}{k_1 \choose m}\right] \\
	& \overset{(b)}{=} -c_0^{k_2} -c_0^{k_1} + c_0^{k_2}.(1) + 1.(c_0^{k_2})\\ 
	& = 0, 
	\end{align*}
where $(b)$ follows from \eqref{eq:ident1} and \eqref{eq:ident2}. 

Coefficient of $\Big<\nabla^2 F(\theta),\Delta \otimes \Delta\Big>$ is given as follows:
\begin{align}
	& \frac{1}{2!}\Big[ c_0^{k_2}\sum_{m=1}^{k_1} \frac{(-1)^{2-m} {k_1 \choose m}}{m}m^2 + c_0^{k_1}\sum_{l=1}^{k_2} \frac{(-1)^{2-l} {k_2 \choose l}}{l}l^2 +  \sum_{l=1}^{k_2} \sum_{m=1}^{k_1} \frac{(-1)^{2-l-m}{k_2 \choose l} {k_1 \choose m}}{lm} (l+m)^2 \Big]. \nonumber 
% & + \sum_{l=1}^k \sum_{m=1}^k \frac{(-1)^{2-l-m}{k \choose l} {k \choose m}}{lm} (l+m)^2 \Big] \nonumber 
\end{align}
The first two terms are $0$ (by Identity \ref{eq:ident3}). Along similar lines as the proof of Lemma \ref{lemma: k1equalk2}, the double summation term can be expanded into three terms. Each of these terms can be written as a product of single summations. The only non-zero term here corresponds to the $2lm$ factor term and turns out to be $1$ (from Identity (\ref{eq:ident2})).  

Coefficient of $\Big<\nabla^q F(\theta), \Delta_\otimes\dots  \Delta\Big>\delta^{q-2}$, $3\leq q \leq k$,  is as follows:
\begin{align}
	& \frac{1}{q!}\Bigg[ c_0^{k_2}\sum_{m=1}^{k_1} \frac{(-1)^{2-m} {k_1 \choose m}}{m}m^q + c_0^{k_1}\sum_{l=1}^{k_2} \frac{(-1)^{2-l} {k_2 \choose l}}{l}l^q + \sum_{l=1}^{k_1} \sum_{m=1}^{k_2} \frac{(-1)^{2-l-m}{k_2 \choose l} {k_1 \choose m}}{lm} (l+m)^q \Bigg] \label{eq:CoeffGenqUneqTrunc}\\
%& = 0 + 0 + 0 \nonumber \\
& = 0. \nonumber
\end{align}
The first two terms in \eqref{eq:CoeffGenqUneqTrunc} are zero from \eqref{eq:ident3}. The third term can be shown to be $0$ along similar lines as the proof of Lemma  \ref{lemma:kth_order}. 
	
We now consider the $q=k+1$ case. Continuing from (\ref{eq:CoeffGenqUneqTrunc}),  one of the first two terms will be zero and the other non-zero, unlike the previous case. 
We consider both these terms (as this algebra will be useful for the range of $q$ we consider next in the proof) and the terms coming from $l^q$ and $m^q$ from the double summation in (\ref{eq:CoeffGenqUneqTrunc}) below first. We obtain
\begin{align}
& c_0^{k_2}\sum_{m=1}^{k_1} (-1)^{2-m} {k_1 \choose m}m^k  + c_0^{k_1}\sum_{l=1}^{k_2} (-1)^{2-l} {k_2 \choose l}l^k   	\left(\sum_{l=1}^{k_2} (-1)^{1-l} {k_2 \choose l} l^k \right) \left(\sum_{m=1}^{k_1} \frac{(-1)^{1-m} {k_1 \choose m}}{m} \right) + \nonumber \\
& \qquad \qquad \qquad \qquad \qquad + \left(\sum_{l=1}^{k_2} \frac{(-1)^{1-l} {k_2 \choose l}}{\ell}\right) \left(\sum_{m=1}^{k_1} (-1)^{1-m} {k_1 \choose m}m^k\right) \label{eq:temp1} \\
& = -c_0^{k_2}\sum_{m=1}^{k_1} (-1)^{1-m} {k_1 \choose m}m^k  - c_0^{k_1}\sum_{l=1}^{k_2} (-1)^{1-l} {k_2 \choose l}l^k + \nonumber  \\    
	& \qquad \qquad \qquad \qquad \qquad  \left [\sum_{l=1}^{k_2} (-1)^{1-l} {k_2 \choose l} l^k \right] c_0^{k_1} + c_0^{k_2} \left [\sum_{m=1}^{k_2} (-1)^{1-m} {k_2 \choose m}m^k\right] \label{eq:temp2_1} \\
	& = 0. \nonumber
\end{align}
The remaining terms of the Binomial expansion from the double summation of  \eqref{eq:CoeffGenqUneqTrunc} can be shown to be $0$ along similar lines as the proof of Lemma~\ref{lemma: k1equalk2}. There arises a single sum of the form $\left(\sum_{l=1}^{k_2} (-1)^{1-l} {k_2 \choose l} l^{p}\right)$, $1\leq p\leq (k-1)$ OR $\left(\sum_{m=1}^{k_1} (-1)^{1-m} {k_1 \choose m} m^{k-1}\right)$, both of which are $0$ from \eqref{eq:ident3}.

The next case that we consider is $k+2 \leq q \leq k'+1$:  Continuing from (\ref{eq:CoeffGenqUneqTrunc}), we consider the first two terms and the terms coming from $l^q$ and $m^q$ from the double summation in (\ref{eq:CoeffGenqUneqTrunc}) first. Note that the first two terms can both be non-zero in this range of $q$ (specifically for $q=k' + 1$). The simplification of these four terms will be exactly similar to (\ref{eq:temp1}) and (\ref{eq:temp2}),  where $m^k$ and $l^k$ must be replaced by $m^{q-1}$ and $l^{q-1}$ respectively and hence will be $0$. 

Next, we consider the remaining terms of the Binomial expansion from the double summation of  (\ref{eq:CoeffGenqUneqTrunc}). 

{\bf{Case 1:}} $\mathbf{k=k_1}$ 

Let us first consider the term ${q \choose 1} lm^{q-1}$. The associated double summation term gives us the following product of single sums.
	\begin{align}
		& \frac{q}{q!} \underbrace{\left(\sum_{l=1}^{k_2} (-1)^{1-l} {k_2 \choose l} \right)}_{1} \left(\sum_{m=1}^{k_1} (-1)^{1-m} {k_1 \choose m}m^{q-2} \right) \nonumber \nonumber \\
		& = \frac{1}{(q-1)!} \left(\sum_{m=1}^{k} (-1)^{1-m} {k \choose m}m^{q-2} \right).  \label{eq:Coeffk1k2}
	\end{align}
	As $q\geq k_1+2=k+2$, the exponent of $m$ is  $\geq k_1=k$. Hence the above sum will be non-zero.  

Consider the remaining $q-2$ terms of the form ${q \choose i} l^im^{q-i}$, for $i=2,\dots,q-1$. The associated typical double summation term gives us the following product of single sums.
\begin{equation}
	{q \choose i} \left[\sum_{l=1}^{k_2} (-1)^{1-l} {k_2 \choose l}l^{i-1} \right] \hspace{-2mm}\left[\sum_{m=1}^{k_1} (-1)^{1-m} {k_1 \choose m}m^{q-i-1} \right]. \nonumber 
	\end{equation}
Consider the $l$ summation above. The exponent of $l$ is at least $1$ as $i\geq2$. The exponent can be at most $q-2$. Since $q\leq k'+1$, i.e.  $q\leq k_2+1$, we have $(q-2) \leq (k_2 - 1)$. Hence the exponent of $l$ is between $1$ and $k_2 - 1$. From Identity \ref{eq:ident3}, this sum is always zero. 

Hence, we are only left with the term in (\ref{eq:Coeffk1k2}). This exactly matches the coefficient in the summation of \eqref{eq:HessianUneqTrunc} with $i$ replaced by $q-2$.

{\bf {Case 2:}} $\mathbf{k=k_2}$

Based on an analogous argument, we can arrive at the same expression as (\ref{eq:Coeffk1k2}) for this case also. We briefly outline this. Here, we need to consider the double summation term associated with $ql^{q-1}m^{1}$, which is as follows. 
\begin{align}
	& \frac{q}{q!} \left(\sum_{l=1}^{k_2} (-1)^{1-l} {k_2 \choose l} l^{q-2}\right) \underbrace{\left(\sum_{m=1}^{k_1} (-1)^{1-m} {k_1 \choose m} \right)}_{1} \nonumber \nonumber \\
		& = \frac{1}{(q-1)!} \left(\sum_{l=1}^{k} (-1)^{1-l} {k \choose l}l^{q-2} \right).  \label{eq:Coeffk1k2Case2}
	\end{align}
Consider the remaining $q-2$ terms of the form ${q \choose j} l^{q-j}m^{j}$, for $j=2,\dots,q-1$.
The associated typical double summation term gives us the following product of single sums.
\begin{equation}
	{q \choose j} \sum_{l=1}^{k_2} (-1)^{1-l} {k_2 \choose l}l^{q-j-1} \times \sum_{m=1}^{k_1} (-1)^{1-m} {k_1 \choose m}m^{j-1}.  \nonumber 
	\end{equation}
The $m$-summation term can be argued to vanish for all $j=2,\dots,q-1$ in exactly the similar manner as the $\ell$-summation in the previous case. Hence, we are left with the same expression as the previous case. This completes the proof.    
\end{proof}

\subsection{Proof of Lemma \ref{lemma:bias_unbalanced}} 
\label{proof:bias_unbalanced}
\begin{proof}
    From Lemma \ref{lemma:kth_order}, we have
    \begin{align}
    &
        \mathcal{H}^k F(\theta) = (\Delta \Delta^T - \mathit{I}) \left[\Delta^T \nabla^2 F(\theta) \Delta\right] + O(\delta^k).\nonumber%\\
        %&\qquad = \frac{9}{\eta^4}M\left[\sum_{i = 1}^d \sum_{j = 1}^d  \Delta^i \Delta^j \nabla^2_{i,j} F(\theta)\right] + O(\delta^k)
    \end{align}
The rest of the proof follows from \cite{prashanth2025gradient} (Chapter 6, Lemma 6.3 ) by taking the expectation on both sides. % and considering the cases (i) $i = j$ and (ii) $i \neq j$. 
By using \ref{ass:noise_bound}, it is easy to see that $\E\|\widetilde{\mathcal{H}}^k - \E[\widetilde{\mathcal{H}}^k|\mathcal{F}_n]\|^2 \leq \E[\|\Tilde{\mathcal{H}}^k\|^2] = O\Big(\frac{1}{\delta^4}\Big).$ 
%where $C_2 = k_1\sigma^2c^l_{\max}\alpha_2^2$
\end{proof}
\subsection{Proof of Theorem \ref{theorem:strong_conv}} 
\label{proof:strong_conv}
\begin{proof} 
%\todoi{add \cite{nesterov2006cubic} for asymptotic convergence of cubic Newton}
The proof follows from Theorem 6.7 in \cite{prashanth2025gradient} by satisfying all assumptions except A.6.11. The assumption A.6.11 follows from Lemma \ref{lemma:bias_unbalanced}.   
\end{proof}  
\subsection{Proof of Theorem \ref{theorem:non_asymptotic}}
\label{proof:non_asymptotic}
The proof proceeds through a sequence of lemmas.
%\begin{proof}
   We follow the proof technique from \cite{balasubramanian2022zeroth}. Unlike \cite{balasubramanian2022zeroth}, we do not assume $f(\cdot,\xi)$ is smooth. Instead, we work with a smooth objective $F$ and this leads to significant deviations in the proof.
%\section{Non-asymptotic bound for CR-Newton method}

\begin{lemma}   \label{lem:mse_grad}
    Let $\bar{G}^k$ be computed as in Algorithm \ref{alg:CR_ZON} with $m_n$ samples. Under assumptions \ref{ass:cont_diff} - \ref{ass:noise_bound}, we have
    \begin{equation}
        \E\Big[\left\|\bar{G}^k - \nabla f(\theta_{n-1},\xi_n)\right\|^2\Big] \leq \frac{2C_2}{m_n \delta^2} + 2C_1^2 \delta^{2k}, 
    \end{equation}
%where $\left\|\E [\mathcal{D}^k f(\theta_{k-1})] - \nabla f(\theta_{k-1})\right\| \leq C_1 \delta^k$ and \; \; $\E \left[\|\mathcal{D}^k f(\theta_{k-1}) - \E[\mathcal{D}^k f(\theta_{k-1})]\|^2\right] \leq \frac{C_2}{\delta^2}$, 
for some constants $C_1, C_2>0$. 
\end{lemma} 
\begin{proof}
From Lemma 1 in \cite{pachalyl2025generalized}, we have  
\begin{align} \label{eq:bias_new}
&
    \left\|\E [\widehat{\mathcal{D}}^k F(\theta_{n-1})] - \nabla F(\theta_{n-1})\right\| \leq C_1 \delta^k \;\;\;  \textrm{and } \nonumber\\
    &\E \left[\|\widehat{\mathcal{D}}^k F(\theta_{n-1}) - \E[\widehat{\mathcal{D}}^k F(\theta_{n-1})]\|^2\right] \leq \frac{C_2}{\delta^2},
\end{align}
 for some constants $C_1, C_2$. 
%\[\left\|\E [\mathcal{D}^k f(\theta_{k-1})] - \nabla f(\theta_{k-1})\right\| \leq C_1 \delta^k\] \[\E \left[\|\mathcal{D}^k f(\theta_{k-1}) - \E[\mathcal{D}^k f(\theta_{k-1})]\|^2\right] \leq \frac{C_2}{\delta^2} \]
Now we calculate
\begin{align*}
    &
    \E[\|\bar{G}^k - \nabla F(\theta_{n-1})\|^2] \\
    & = \E[\|\bar{G}^k - \E[\bar{G}^k] + \E[\bar{G}^k] - \nabla F(\theta_{n-1})\|^2] \\
    & \overset{(a)}{\leq} 2\Big(\E[\|\bar{G}^k - \E[\bar{G}^k]\|^2 + \|\E[\bar{G}^k] - \nabla F(\theta_{n-1})\|^2] \Big) \\
    & \overset{(b)}{\leq} \frac{2}{m_n^2}\frac{C_2}{\delta^2}m_n + 2C_1^2 \delta^{2k} \\
    & = \frac{2C_2}{m_n \delta^2} +2 C_1^2 \delta^{2k}. 
\end{align*} 
In the above, (a) follows from the fact that $\|x+y\|^2 \leq 2(\|x\|^2 + \|y\|^2)$ and (b) follows from \eqref{eq:bias_new}.  
\end{proof}

\begin{lemma} \label{lem:cubic}
    Let $\bar{\mathcal{H}}^k$ be computed as in Algorithm \ref{alg:CR_ZON} with $b_n$ samples and  let $b_n \geq C(d) = 4(1 + 2 \log 2d)$. Then under assumptions \ref{ass:cont_diff} - \ref{ass:noise_bound}, we have
    %where $H_{\delta} (\theta_{k-1}, \xi, \Delta_i)$ is defined in () and $b_k \geq 4(1 + 2\log 2d)$. Under assumptions (A1) - (A1), we have 
    \begin{align*}
    &
        \E \left[\|\bar{{\mathcal{H}}}^k - \nabla^2 F(\theta_{n-1})\|^3\right] \leq
        \sqrt{\Big(\frac{4C_4C(d)}{b_n \delta^4} + 2C_3^2\delta^{k}\Big)\Big(\frac{24C_4^2}{b_n^2 \delta^8} + 8C_3^4 \delta^{2k}\Big)},  
    \end{align*} 
    for some constants $C_3, C_4$. 
%    where $\|\E[a-b] \|\leq C_3 \delta^k$ and $\E[\|c-d\|^2] \leq \frac{C_4}{\delta^4}$. 
\end{lemma}
\begin{proof}
%    See Sec. []. 
From Lemma \ref{lemma:bias_unbalanced}, we have 
\begin{align}
    &
    \Big\|\E\Big[\Tilde{\mathcal{H}}^k F(\theta,\xi)\Big] - \nabla^2 F(\theta)\Big\| \leq C_3 \delta^k  \;\;\; \textrm{and} \nonumber \\
    & \E\Big[\|\Tilde{\mathcal{H}}^k F(\theta,\xi) -\E[\Tilde{\mathcal{H}}^k F(\theta,\xi)]\|^2\Big ] \leq \frac{ C_4}{\delta^4}, \label{eq: bias_var} 
\end{align}
for some constants $C_3, C_4$. 
Note that 
    \begin{align} 
    &
        \E\left[\|\bar{\mathcal{H}}^k - \nabla^2 F(\theta_{n-1})\|^2\right] \nonumber \\
        & = \E\Big[\|\bar{\mathcal{H}}^k -\E[\Tilde{\mathcal{H}}^kF(\theta_{n-1},\xi_{i,n})]+ \E[\Tilde{\mathcal{H}}^kF(\theta_{n-1},\xi_{i,n})] - \nabla^2 F(\theta_{n-1})\|^2\Big] \nonumber\\
        & \leq 2\Big( \E\Big[\|\bar{\mathcal{H}}^k -\E[\Tilde{\mathcal{H}}^k F(\theta_{n-1},\xi_{i,n})]\|^2 + \E[\|[\E[\Tilde{\mathcal{H}}^k F(\theta_{n-1},\xi_{i,n})] - \nabla^2 F(\theta_{n-1})]\|^2\Big]\Big) \nonumber\\ 
        & \leq 2\E\Big[\|\bar{\mathcal{H}}^k -\E[\Tilde{\mathcal{H}}^k F(\theta_{n-1},\xi_{i,n})]\|^2\Big] + 2C_3^2 \delta^{2k}. \label{eq: var_tropp}
    \end{align} 
We know from Theorem 1 in \cite{tropp2016expected} that 
\begin{align*}
&
    \E\Big[\|\bar{\mathcal{H}}^k - \E[\Tilde{\mathcal{H}}^k F(\theta_{n-1},\xi_{i,n})]\|^2\Big] \leq \frac{2C(d)}{b_n^2}\Big(\Big\|\sum_{i=1}^{b_n}\E[\Delta_{i,n}^2]\Big\|  + C(d) \E\Big[\max_{i} \|\Delta_{i,n}\|^2\Big]\Big),
\end{align*}
where $\Delta_{i,n} = \Tilde{\mathcal{H}}^k F(\theta_{n-1},\xi_{i,n}) - \E[\Tilde{\mathcal{H}}^k F(\theta_{n-1},\xi_{i,n})]$ and $C(d) = 4(1+2\log 2d)$. 
Using (\ref{eq: bias_var}), we have $\E[\|\Delta_{i,n}\|^2] \leq \frac{C_4}{\delta^4}$. 
Thus from (\ref{eq: var_tropp}), we get
\begin{align}
     \E\Big[\|\bar{\mathcal{H}}^k - \nabla^2 F(\theta_{k-1})\|^2\Big] \leq \frac{4C_4 C(d)}{b_n\delta^4} + 2C_3^2 \delta^{2k}.  \label{eq:mse_2}
\end{align} 
%For notational simplicity, we set $\Tilde{\mathcal{H}}^k(\theta_{n-1}) \equiv \Tilde{\mathcal{H}}^k(\theta_{n-1},\xi_n)$ and $\Delta_{n} \equiv \Delta_{i,n}$. 
Now 
\begin{align} 
    &
        \E\Big[\|\bar{\mathcal{H}}^k - \nabla^2 F(\theta_{n-1})\|^4\Big] \nonumber \\
        & = \E[\|\bar{\mathcal{H}}^k -\E[\Tilde{\mathcal{H}}^k(\theta_{n-1},
        \xi_{i,n})]  + \E[\Tilde{\mathcal{H}}^k(\theta_{n-1},\xi_{i,n}) - \nabla^2 F(\theta_{n-1})\|^4] \nonumber\\
        & \overset{(c)}{\leq} 8\Big( \E\Big[\|\bar{\mathcal{H}}^k -\E[\Tilde{\mathcal{H}}^k(\theta_{n-1},\xi_{i,n})\|^4 +\|\E[\Tilde{\mathcal{H}}^k(\theta_{n-1},\xi_{i,n}) \hspace{-1mm}- \hspace{-1mm}\nabla^2 F(\theta_{n-1})]\|^4\Big]\Big) \nonumber\\
        & \overset{(d)}{\leq} 8\E\Big[\|\bar{\mathcal{H}}^k -\E[\Tilde{\mathcal{H}}^k(\theta_{n-1},\xi_{i,n})\|^4 \Big]+ 8C_3^4 \delta^{4k}. \label{eq:bound_4}
    \end{align}  
Here (c) follows from the fact that $\|x+y\|^4 \leq 8(\|x\|^4 + \|y\|^4)$ and (d) follows from \eqref{eq: bias_var}. Now consider   
\begin{align*}
&
    \E\Big[\Big\|\bar{\mathcal{H}}^k -\E[\Tilde{\mathcal{H}}^kF(\theta_{n-1},\xi_{i,n})]\Big\|_F^4\Big]\\
    & = \E\Big[\Big\|\frac{1}{b_n}\sum_{i=1}^{b_n} \Tilde{\mathcal{H}}^k F(\theta_{n-1},\xi_{i,n}) - \E[\Tilde{\mathcal{H}}^k (\theta_{n-1},\xi_{i,n})]\Big\|_F^4\Big]\\
    & \overset{(e)}{=}\frac{1}{b_n^4} \E\Big[\Big\|\sum_{i=1}^{b_n} \Big[\Tilde{\mathcal{H}}^k(\theta_{n-1},\xi_{i,n}) - \E[\Tilde{\mathcal{H}}^k(\theta_{n-1},\xi_{i,n})]\Big]\Big\|_F^4\Big]\\
    & \overset{(f)}{\leq} \frac{1}{b_n^4} 3b_n^2\E\Big[\|\Delta_{i,n}\|^4\Big] \leq \frac{3C_4^2}{\delta^8 b_n^2}, 
\end{align*}  
where $\|.\|_F$ denotes the Frobenius norm.  In the above, (e) follows from the norm property and (f) follows from Lemma 16 in \cite{maniyar2024cubic}. 
Thus from \eqref{eq:bound_4}, we have 
\begin{align}
    \E\Big[\|\bar{\mathcal{H}}^k - \nabla^2 F(\theta_{n-1})\|^4\Big] = \frac{24C_4^2}{\delta^8 b_n^2} + 8C_3^4\delta^{4k}. \label{eq:mse_4}
\end{align}
Now 
\begin{align*}
    &
    \E\Big[\Big\|\bar{H}^k - \nabla^2 F(\theta_{n-1})\Big\|^3\Big] \\
    & \overset{(g)}{\leq} \Bigg(\E\Big[\Big\|\bar{H}^k - \nabla^2 F(\theta_{n-1})\Big\|^2\Big]\E\Big[\Big\|\bar{H}^k - \nabla^2 F(\theta_{n-1})\Big\|_F^4\Big]\Bigg)^{\frac{1}{2}} \\
    &  \overset{(h)}{\leq} \sqrt{\Big(\frac{4C_4C(d)}{b_n \delta^4} + 2C_3^2\delta^{2k}\Big)\Big(\frac{24C_4^2}{b_n^2 \delta^8} + 8C_3^4 \delta^{4k}\Big)}\\ 
    & \overset{(i)}{\leq} \sqrt{\Big(\frac{4C_4C(d)}{b_n \delta^4} + 2C_3^2\delta^{k}\Big)\Big(\frac{24C_4^2}{b_n^2 \delta^8} + 8C_3^4 \delta^{2k}\Big)}. 
\end{align*} 
In the above, (g) follows from the Holder inequality, (h) follows from \eqref{eq:mse_2} and \eqref{eq:mse_4}, and (i) holds as $\delta<1$. 
\end{proof}
%\todoi{State Lemma 9 and Lemma 10 here.} 
Now we state the following two lemmas from \cite{balasubramanian2022zeroth} to prove the main result. 
\begin{lemma} \label{lemma: lemma9}
    Let the iterates $\{\theta_n\}$ be computed by Algorithm \ref{alg:CR_ZON}. Then we have 
    \begin{align*}
    &
        \sqrt{\E\Big[\|\theta_n - \theta_{n-1}\|^2\Big]} \geq \max \Bigg\{\sqrt{\frac{\E[\|\nabla F(\theta_{n-1})\|] - \delta^g_n - \delta^{\mathcal{H}}_n}{L_{\mathcal{H}} + \alpha_k}}, \frac{-2}{\alpha_n + L_{\mathcal{H}}}[\E[\lambda_{\min}(\nabla^2 F(\theta_n))] + \sqrt{2({\alpha_n + L_{\mathcal{H}})\delta^{\mathcal{H}}_n}}]\Bigg\},
    \end{align*}
    where $ \left\|\E [\bar{G}^k F(\theta_{n-1})] - \nabla F(\theta_{n-1})\right\| \leq (\delta^g_n)^2$ and $\E \left[\|\bar{\mathcal{H}}^k \hspace{-1mm}- \hspace{-1mm}\nabla^2 F(\theta_{n-1})\|^3\right] \leq [2(\alpha_n + L_{\mathcal{H}})\delta^{\mathcal{H}}_n]^{3/2}$, for some $\delta_n^g, \delta_n^{\mathcal{H}} > 0$. 
\end{lemma} 
\begin{proof}
    See Lemma 9 in  \cite{balasubramanian2022zeroth}. 
\end{proof}
\begin{lemma} \label{lemma: lemma10}
    Let the iterates $\{\theta_n\}_{n=1}^N$ be computed by Algorithm \ref{alg:CR_ZON}. Then we have
    \begin{align*} 
    &
        \E\Big[\|\theta_R - \theta_{R-1}\|^3\Big] \leq \frac{36}{\sum_{k=1}^N\alpha_n} \Big[F(\theta_0) - F^* \\
        & + \sum_{k=1}^N \frac{4(\delta_n^g)^{3/2}}{\sqrt{3\alpha_n}} + \sum_{k=1}^N\left(\frac{18 \sqrt[4]{2}}{\alpha_n}\right)^2 \left((L_{\mathcal{H}} + \alpha_n)\delta_n^{\mathcal{H}}\right)^{3/2}\Big], 
    \end{align*} 
    where $\theta_R$ is chosen uniformly at random from $\{\theta_1, \theta_2, \ldots, \theta_N\}$. 
\end{lemma} 
\begin{proof}
    See Lemma 10 in \cite{balasubramanian2022zeroth}. 
\end{proof} 
% \begin{table}[h]
% \centering
% \captionsetup{format = plain}
%  \caption{Parameter error ($\pm$ standard deviation) for GSF and G2SF algorithms under the Rastrigin objective defined in \eqref{eq:rastrig} with budget $5,000$ and $10,000$.  GSF-5 corresponds to the five measurements GSF proposed in \cite{pachalyl2025generalized}, and G2SF-9 represents our nine measurements Hessian estimator. 
%  }
% \label{tab:comp_first_second}
%  \begin{subtable}{0.36\textwidth}
% 		\centering
% \begin{tabular}{|c|c|c|}
% \toprule
% \rowcolor{gray!20}
% \multicolumn{3}{|c|}{\multirow{2}{*}{\textbf{ Budget $ 5 \times 10^3$}} }\\[1em]
% %&&&\\
% \midrule
%   & \multirow{2}{*}{GSF - 5 } & \multirow{2}{*}{G2SF - 9}  \\
%    & &   \\
%  \midrule
% $d=5$ & $0.35\pm 0.13$ & $\bf{0.15\pm 0.07} $  \\
% && \\
% $d=10$  & $0.41 \pm 0.09 $& $\bf{0.22\pm 0.06}$   \\
% && \\
% $d=50$  & $0.61 \pm 0.09 $& $\bf{0.35 \pm 0.04}$   \\
% && \\
% $d=100$  & $0.81\pm 0.09 $& $\bf{0.44\pm 0.11}$   \\
%  \bottomrule
%  \end{tabular}
%  \end{subtable}
% \hspace*{1ex}
%  \begin{subtable}{0.185\textwidth}
%  		\centering
%  \begin{tabular}{|c|c|}
%  \toprule
%  \rowcolor{gray!20}
%  \multicolumn{2}{|c|}{\multirow{2}{*}{\textbf{Budget $1 \times 10^4$}}}\\[1em]
%  \midrule
%    \multirow{2}{*}{GSF-5 } & \multirow{2}{*}{G2SF-9}   \\
%    & \\
%  \midrule
%  $0.40  \pm 0.15$ & $\bf{0.15 \pm 0.17} $  \\
%  &\\
%   $0.41 \pm 0.11 $& $\bf{0.15 \pm 0.05}$  \\
%  &\\
%   $0.48 \pm 0.04 $& $\bf{0.37 \pm 0.17}$  \\
%    &\\
%      $0.67 \pm 0.07 $& $\bf{0.40 \pm 0.13}$  \\ 
%   \bottomrule
%  \end{tabular}
%  \end{subtable} 
% \end{table} 
We now prove the main claim in Theorem \ref{theorem:non_asymptotic}.

\begin{proof}\textit{\textbf{(Theorem \ref{theorem:non_asymptotic})}}
With the choice of hyperparameters defined in \eqref{eq:hyper}, Lemma \ref{lem:mse_grad} and Lemma \ref{lem:cubic} are satisfied with $\delta^g_n = \delta_n^{\mathcal{H}} = \epsilon$. Now from Lemma \ref{lemma: lemma10}, we obtain
\begin{align}
%&
    \E[\|\theta_R - \theta_{R-1}\|^3] &\leq \frac{12}{L_{\mathcal{H}}} \Big [\frac{F(\theta_0) - F(\theta^*)}{N} + \frac{4\epsilon^{3/2}}{3\sqrt{L_{\mathcal{H}}}} + \frac{288\sqrt{2}\epsilon^{3/2}}{\sqrt{L_{\mathcal{H}}}} \Big ] \\
    & \leq \frac{5000\epsilon^{3/2}}{L_{\mathcal{H}}^{3/2}}. 
\end{align} 
Using Lyapunov inequality, we have
\begin{align*}
    [\E[\|\nabla F(\theta_R - \theta_{R-1})\|^2]]^{1/2} &\leq [\E[\|\nabla F(\theta_R - \theta_{R-1})\|^3]]^{1/3}\\
    &\leq \frac{5000^{1/3}\epsilon^{1/2}}{L_{\mathcal{H}}^{1/2}}. 
\end{align*} 
From Lemma \ref{lemma: lemma9}, we obtain
\begin{align*} 
%&
    \sqrt{\|\E[\nabla F(\theta_R)]\|} \leq 35\sqrt{\epsilon},
%    &
    \E\Big[\frac{-\lambda_{\min}(\nabla^2 F(\theta_R))}{\sqrt{L_{\mathcal{H}}}}\Big]  \leq 50 \sqrt{\epsilon}. 
\end{align*} 
%\frac{-2}{\alpha_n + 2 L_{\mathcal{H}}}
The last inequality follows from the below, which was given in the proof of Lemma 9 in \cite{balasubramanian2022zeroth}. 
\begin{align}
&
    \sqrt{\E[\|\theta_n - \theta_{n-1}\|^2]} \geq\frac{-2}{\alpha_n + 2L_{\mathcal{H}}} \left[\sqrt{2(\alpha_n + L_{\mathcal{H}})\delta_n^{\mathcal{H}}} + \E[\lambda_{\min}(\nabla^2 F(\theta_n))]\right]. 
\end{align} 
Thus 
\begin{equation}
        35\sqrt{\epsilon} \geq \max \left\{\sqrt{\E[\|\nabla F(\theta_R)\|]}, \frac{-1}{50\sqrt{L_\mathcal{H}}}\E[\lambda_{\min}(\nabla^2 F(\theta_R))] \right\}. 
    \end{equation}  
This concludes the proof. 
\end{proof} 
% Using Lyapunov inequality, we have
% \begin{align*}
%     [\E[\|\nabla F(\theta_R - \theta_{R-1})\|^2]]^{1/2} &\leq [\E[\|\nabla F(\theta_R - \theta_{R-1})\|^3]]^{1/3}\\
%     &\leq \frac{5000^{1/3}\epsilon^{1/2}}{L_{\mathcal{H}}^{1/2}}. 
% \end{align*} 
% From Lemma \ref{lemma: lemma9}, we obtain
% \begin{align*} 
%     \sqrt{\|\E[\nabla F(\theta_R)]\|} &\leq 35\sqrt{\epsilon} \\
%     \E[\frac{-\lambda_{\min}(\nabla^2 F(\theta_R))}{\sqrt{L_{\mathcal{H}}}}] & \leq 50 \sqrt{\epsilon}. 
% \end{align*}  
%\todoi{Add the remaining proof. } 
% \begin{theorem}
%     Let the iterate $\{\theta_k\}$ be calculated in Algo (). The parameters are set as follows: 
%     \[\alpha_k = \mathcal{L}_H; \;\;\; \delta = O(\epsilon^{\frac{1}{2k}}); \;\;\; N = O(\frac{1}{\epsilon^{\frac{3}{2}}})\] \[ m_k = O(\frac{1}{\epsilon^{2+\frac{1}{k}}});\;\;\; b_k = O(\frac{1}{\epsilon^{1+\frac{2}{k}}}) .\] Then we have 
%     \begin{equation}
%         5\sqrt{\epsilon} \geq \max \{a,b \}. 
%     \end{equation}
% \end{theorem} 
% \begin{proof}
%     Later. We follow the proof technique from \cite{balasubramanian2022zeroth}. 
% \end{proof} 
%\todoi{have to complete today} 
%\end{proof}
 
\section{Simulation Experiments}
\label{sec:simul}
In this section, we implement the stochastic Newton method using our proposed Hessian estimators with different values of $k$ to evaluate the performance on the well-known Rastrigin objective defined below.
We are interested in solving the following $d$-dimensional minimization problem:
\begin{equation}
    \min_{\theta \in \mathbb{R}^d} F(\theta).
\end{equation}
The exact form of $F(\theta)$ is not known, instead, a noisy observation $f(\theta,\xi) = F(\theta) + \xi$ is made available. We take the noise term $\xi$ as $[\theta^T,1]z$, where $z$ is distributed over a $(d+1)$-dimensional multivariate Gaussian with zero mean and covariance $\sigma^2 \mathcal{I}_{d+1}$. In our experiments, we take $\sigma = 0.001$. 

\textbf{Rastrigin Function:} We use the $d$-dimensional Rastrigin objective function in our experiment, which is defined as follows: 
\begin{equation} 
\label{eq:rastrig}
    f(\theta,\xi) = 10d + \sum_{n = 1}^d \Big[\theta_n^2 - 10\cos{(2\pi \theta_n)}\Big] + \xi, 
\end{equation}
where $\xi$ is defined as before. The optimal value $\theta^*$ for the Rastrigin function is the $d$-dimensional vector of zeros. 

We use the following metric similar to \cite{prashanth2016adaptive}, \cite{pachalyl2025generalized} to test the performance of our proposed method: 
\begin{equation}
    \textrm{Parameter error} = \frac{\|\theta_{T}- \theta^*\|^2}{\|\theta_0 - \theta^*\|^2}, 
\end{equation}
where $T$ is the final epoch of the simulation experiments. Note that $T$ is different for different Hessian estimators for a given fixed budget (total number of function measurements). Thus, methods that have lower parameter errors would be preferred over others.

We implement the following algorithms: 
%\textbf{2SPSA:} here
\begin{enumerate}
    \item \textbf{2SPSA :} This is an SPSA-based Newton method studied in \cite{spall2000adaptive}, which requires four function measurements per iteration. Here, the perturbation parameters are chosen from a symmetric $\pm 1$-valued Bernoulli distribution. 
    \item \textbf{2RDSA-Unif :} An RDSA-based adaptive Newton method proposed in \cite{prashanth2016adaptive} with independent uniform random variables as random perturbations. This method requires three function measurements per iteration.  
    \item \textbf{G2SF:} This corresponds to the adaptive Newton method with our proposed Gaussian perturbation based Hessian estimators described in \ref{sec:unequal_truncation}. 
    \item \textbf{GSF:} This is a gradient-based method with Gaussian perturbation parameters proposed in \cite{pachalyl2025generalized}. 
\end{enumerate} 
 
\textbf{Experimental setup: } In our experiments, we set $\delta(n) = \frac{0.9}{n^{0.16667}}$, $a(n) = \frac{0.9}{(n+20)^{0.9}} $, and $b(n) = \frac{0.9}{(n+10)^{0.56}} $. We test the performance of our proposed algorithm for $d = 5, 10, 50, 100$. All the results are shown after taking average over $10$ independent runs. 

\textbf{Experimental results: } We compare the second-order methods with the first-order methods in Table \ref{tab:comp_first_second} with varied number of function measurements. This table shows that the second-order methods outperform the first-order methods. In Tables \ref{tab:GSF_3_9} and \ref{tab:G2RDSA_3_9}, we implement our proposed methods with different measurements with Gaussian and uniform perturbation parameters. This clearly shows that the second-order methods with more function measurements do perform well. This validates our theoretical findings. Finally, we compare our proposed methods with the existing methods including \cite{spall2000adaptive}, \cite{prashanth2016adaptive} and observe superior performance of these in Table \ref{tab:comp_all}.       

\begin{table}[h]
\centering
\captionsetup{format = plain}
 \caption{Parameter error ($\pm$ standard deviation) for GSF and G2SF algorithms under the Rastrigin objective defined in \eqref{eq:rastrig} with budget $5,000$ and $10,000$.  GSF-5 corresponds to the five measurements GSF proposed in \cite{pachalyl2025generalized}, and G2SF-9 represents our nine measurements Hessian estimator. 
 }
\label{tab:comp_first_second}
 \begin{subtable}{0.36\textwidth}
		\centering
\begin{tabular}{|c|c|c|}
\toprule
\rowcolor{gray!20}
\multicolumn{3}{|c|}{\multirow{2}{*}{\textbf{ Budget $ 5 \times 10^3$}} }\\[1em]
%&&&\\
\midrule
  & \multirow{2}{*}{GSF - 5 } & \multirow{2}{*}{G2SF - 9}  \\
   & &   \\
 \midrule
$d=5$ & $0.35\pm 0.13$ & $\bf{0.15\pm 0.07} $  \\
&& \\
$d=10$  & $0.41 \pm 0.09 $& $\bf{0.22\pm 0.06}$   \\
&& \\
$d=50$  & $0.61 \pm 0.09 $& $\bf{0.35 \pm 0.04}$   \\
&& \\
$d=100$  & $0.81\pm 0.09 $& $\bf{0.44\pm 0.11}$   \\
 \bottomrule
 \end{tabular}
 \end{subtable}
\hspace*{1ex}
 \begin{subtable}{0.185\textwidth}
 		\centering
 \begin{tabular}{|c|c|}
 \toprule
 \rowcolor{gray!20}
 \multicolumn{2}{|c|}{\multirow{2}{*}{\textbf{Budget $1 \times 10^4$}}}\\[1em]
 \midrule
   \multirow{2}{*}{GSF-5 } & \multirow{2}{*}{G2SF-9}   \\
   & \\
 \midrule
 $0.40  \pm 0.15$ & $\bf{0.15 \pm 0.17} $  \\
 &\\
  $0.41 \pm 0.11 $& $\bf{0.15 \pm 0.05}$  \\
 &\\
  $0.48 \pm 0.04 $& $\bf{0.37 \pm 0.17}$  \\
   &\\
     $0.67 \pm 0.07 $& $\bf{0.40 \pm 0.13}$  \\ 
  \bottomrule
 \end{tabular}
 \end{subtable} 
\end{table} 

\begin{table*}[ht] 
\centering 
 \caption{Parameter error ($\pm$ standard deviation) for our algorithm under the Rastrigin objective defined in \eqref{eq:rastrig} with three and nine measurement Hessian estimators G2SF-3 and G2SF-9, respectively, and with budget of 2,000 and 30,000.   
 }
\label{tab:GSF_3_9}
 \begin{subtable}{0.37\textwidth}
		\centering
\begin{tabular}{|c|c|c|}
\toprule
\rowcolor{gray!20}
\multicolumn{3}{|c|}{\multirow{2}{*}{\textbf{ Budget $ 2 \times 10^3$}} }\\[1em]
\midrule
  & \multirow{2}{*}{G2SF - 3} & \multirow{2}{*}{G2SF - 9}  \\
   & &   \\ 
 \midrule
$d = 5$ & $0.40 \pm 0.14$ & $\bf{0.32 \pm 0.12} $  \\
&& \\
$d = 10$  & $0.46\pm0.15 $& $\bf{0.31\pm 0.15}$   \\
&& \\
$d = 50$  & $0.48 \pm 0.10 $& $\bf{0.43 \pm 0.05}$   \\
&& \\
$d = 100$  & $\bf{0.45 \pm 0.05} $& $\bf{0.45 \pm 0.04}$   \\
 \bottomrule
 \end{tabular}
 \end{subtable}
\hspace*{1ex}
 \begin{subtable}{0.13\textwidth}
 		\centering
 \begin{tabular}{|c|c|}
 \toprule
 \rowcolor{gray!20}
 \multicolumn{2}{|c|}{\multirow{2}{*}{\textbf{Budget $3 \times 10^4$}}}\\[1em]
 \midrule
   \multirow{2}{*}{G2SF - 3} & \multirow{2}{*}{G2SF - 9}   \\
   & \\
 \midrule
 $0.55 \pm0.10$ & $\bf{0.11 \pm 0.28} $  \\
 &\\
  $0.57 \pm 0.18 $& $\bf{0.17 \pm 0.16}$  \\
 &\\
  $0.52 \pm 0.14 $& $\bf{0.34 \pm 0.15}$  \\
   &\\
     $0.48 \pm 0.08 $& $\bf{0.34 \pm 0.10}$  \\ 
  \bottomrule
 \end{tabular}
 \end{subtable}
 
\end{table*}  
\centering 
\begin{table*}[ht]
\captionsetup[subtable]{position = below}
	\captionsetup[table]{position=top}
 \caption{Parameter error ($\pm$ standard deviation) under the Rastrigin objective defined in \eqref{eq:rastrig} with three measurement G2RDSA (G2R-3) and nine measurement G2RDSA(G2R-9) Hessian estimators, with budget of 5,000 and 10,000.
 }
\label{tab:G2RDSA_3_9} 
\centering 
 \begin{subtable}{0.36 \textwidth}
		\centering
\begin{tabular}{|c|c|c|}
\toprule
\rowcolor{gray!20}
\multicolumn{3}{|c|}{\multirow{2}{*}{\textbf{ Budget $ 5 \times 10^3$}} }\\[1em]
\midrule
  & \multirow{2}{*}{G2R-3} & \multirow{2}{*}{G2R-9}  \\
   & &   \\
 \midrule
$d = 5$ & $0.52 \pm 0.20$ & $\bf{0.20 \pm 0.14} $  \\
&& \\
$d = 10$  & $0.52 \pm0.13 $& $\bf{0.29\pm 0.11}$   \\
&& \\
$d = 50$  & $0.46 \pm0.10 $& $\bf{0.44 \pm 0.05}$   \\
&& \\
$d = 100$  & $0.48\pm 0.06 $& $\bf{0.45 \pm0.03}$   \\
 \bottomrule
 \end{tabular}
 \end{subtable}
\hspace*{1ex}
 \begin{subtable}{0.16\textwidth}
 		\centering
 \begin{tabular}{|c|c|}
 \toprule
 \rowcolor{gray!20}
 \multicolumn{2}{|c|}{\multirow{2}{*}{\textbf{Budget $1 \times 10^4$}}}\\[1em]
 %&&&\\
 \midrule
   \multirow{2}{*}{G2R - 3} & \multirow{2}{*}{G2R - 9}   \\
   & \\
 \midrule
 $0.63 \pm 0.19$ & $\bf{0.22 \pm 0.13} $  \\
 &\\
  $0.47\pm 0.14 $& $\bf{0.26 \pm0.06}$  \\
 &\\
  $0.45 \pm0.07 $& $\bf{0.37 \pm0.05}$  \\
   &\\
     $0.51 \pm0.11 $& $\bf{0.41 \pm0.02}$  \\ 
  \bottomrule
 \end{tabular}

 \end{subtable}
 
\end{table*}   
\begin{table*}
\centering
\captionsetup[subtable]{position = below}
	\captionsetup[table]{position=top}
 \caption{ Parameter error ($\pm$ standard deviation) comparison with existing approaches for the Rastrigin objective defined in \eqref{eq:rastrig} for $d = 10$ and budget of 5,000 and 10,000.  
 }
\label{tab:comp_all}
 \begin{subtable}{0.55\textwidth}
% \centering
\begin{tabular}{|c|c|c|c|c|}
\toprule
\rowcolor{gray!20}
\multicolumn{5}{|c|}{\multirow{2}{*}{Comparison with existing methods} }\\[1em]
\midrule
  \multirow{2}{*}{Budget}  & \multirow{2}{*}{2SPSA\cite{spall2000adaptive}} & \multirow{2}{*}{2RDSA\cite{prashanth2016adaptive}} & \multirow{2}{*}{G2R-9} & \multirow{2}{*}{G2SF-9} \\
  && &&\\
%    &\cite{spall2000adaptive} & \cite{prashanth2016adaptive} &(ours) &(ours)\\
 \midrule
$ 5 \times 10^3$ & $0.58 \pm 0.15$ & $0.44 \pm 0.14 $ & $\bf{0.29 \pm 0.11}$&$\bf{0.22 \pm 0.06}$ \\
&& &&\\
 $ 1 \times 10^4$  & $0.55 \pm 0.18$ & $0.45 \pm 0.14$ &$\bf{0.26 \pm 0.06}$ & $\bf{0.15 \pm 0.05}$\\
%&& &\\ 
%$d = 50$  & $.46 \pm.10 $& $.44 \pm .05$ &  \\
%&& &\\
%$d = 100$  & $.48\pm .06 $& $.45 \pm.03$ &  \\
 \bottomrule
 \end{tabular}
 \end{subtable}
\end{table*}  

\section{Conclusions} 
\label{sec:con} 
We proposed a family of Hessian estimators that utilize only function measurements and demonstrated that estimators with more function measurements yield lower-order bias. We analyzed the asymptotic convergence of the stochastic Newton method with our proposed estimators. Also, we established a non-asymptotic bound for our proposed algorithm, which is efficient for escaping saddle points. 

As future work, it would be interesting to explore the stochastic Newton methods with our proposed estimators in real-world applications such as natural language processing and navigation in autonomous systems.

%\bibliographystyle{plain}        % Include this if you use bibtex 
%\bibliography{refs_Autom}           % and a bib file to produce the 
%\clearpage 
\bibliographystyle{plain}
\bibliography{refs_Autom}

\end{document}